%% file: iclr2027_conference.tex
\documentclass{article} 

\usepackage{iclr2027_conference,times}

\usepackage[utf8]{inputenc} 
\usepackage[T1]{fontenc}    
\usepackage[hidelinks]{hyperref}
\usepackage{url}
\usepackage{microtype}
\usepackage{graphicx}
\usepackage{caption}
\usepackage{subcaption}
\usepackage{booktabs}
\usepackage{multirow}
\usepackage{etoc}
\usepackage{titletoc}       
\usepackage{algorithm}
\usepackage{algorithmic}
\usepackage{amsmath}
\usepackage{amssymb}
\usepackage{amsfonts}
\usepackage{mathtools}
\usepackage{amsthm}
\newtheorem{theorem}{Theorem}

\usepackage[capitalize,noabbrev]{cleveref}
\usepackage{float}

\newcommand{\rv}[1]{#1} 

\newcommand{\method}{GRAB}
\newcommand{\methodfull}{Group-Relative Advantage Balancing}

\title{\rv{When Rewards Disagree: Reward Weighting for Token-Efficient Video Reasoning}}

\author{Zinuo Li$^{1,3}$, Yongxin Guo$^{2}$, Jun Liu$^{3}$, Jiawei Zhan$^{3}$, Xi Jiang$^{4}$, \\
  \textbf{Mohammed Bennamoun$^{1}$, Farid Boussaid$^{1}$, Feng Zheng$^{4}$, Qiuhong Ke$^{5*}$} \\
  {$^{1}$University of Western Australia}
  {$^{2}$CUHKSZ}
  {$^{3}$Tencent Youtu Lab}
  {$^{4}$SUSTech}
  {$^{5}$Monash University} \\
}

\preprint
\begin{document}

\maketitle

\vspace{-2.5em}

\begin{center}
    \includegraphics[width=\textwidth]{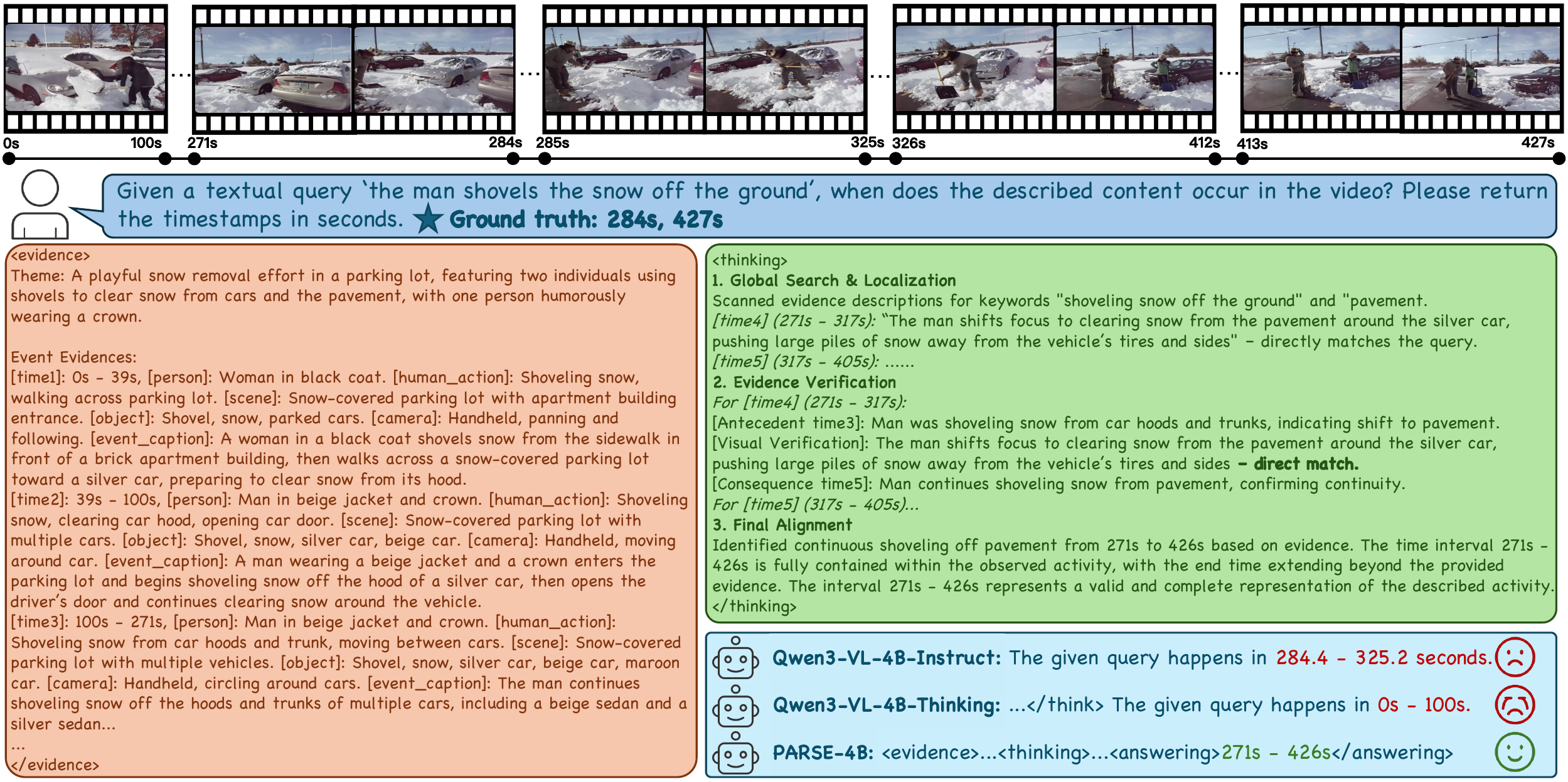}
    \vspace{-1em}
    \captionof{figure}{\rv{Overview of the proposed structured video reasoning framework. The model first organizes the video into temporally ordered, event-level evidence describing salient entities, actions, scenes, objects, and other relevant cues. This structured representation then constrains the subsequent reasoning process, allowing the model to focus on the relevant temporal dependencies and produce more grounded and interpretable predictions.}}
    \label{fig:teaser}
\end{center}
\vspace{-0.5em}

\begin{abstract}
Human perception interprets video dynamics by constructing structured mental representations of entities, actions, and temporal boundaries before engaging in high-level reasoning. In contrast, current Video-LLMs typically rely on unstructured Chain-of-Thought (CoT), where crucial visual cues are obscured by verbose narration. While recent efforts ground reasoning in intermediate visual evidence, they rely on task-specific schemas that fail to generalize across diverse video understanding tasks. To address this, we introduce \textbf{Structured Event Evidence}, a unified, temporally ordered schema designed for general video reasoning. However, post-training a model to extract structured evidence while reasoning introduces a severe multi-objective optimization bottleneck, as candidate rollouts must simultaneously satisfy schema compliance, task accuracy, and reasoning budget. Standard reinforcement learning (RL) algorithms attempt to balance these objectives via static scalarization, enforcing a rigid accuracy--length trade-off across training and obscuring fine-grained reward signals. We resolve this by proposing \textbf{Group-Relative Advantage Balancing (GRAB)}, which dynamically computes Pareto-optimal objective weights within each rollout group via a minimum-norm criterion over centered rewards. Operating strictly in reward space, GRAB incurs virtually zero computational overhead ($<0.002\%$ relative runtime to GRPO) and eliminates brittle weight tuning. Supported by our new \textbf{PARSE-60K} dataset (60K annotations across 35K videos) and a three-stage progressive curriculum, we train \textbf{PARSE-4B}. Across temporal grounding and video question-answering benchmarks, GRAB generates \textbf{13\% to 39\% fewer tokens} while maintaining competitive accuracy with exhaustively hand-tuned static weights—without requiring hyperparameter weight sweeps.
\end{abstract}
\input{sec/1_intro}

\input{sec/2_related}
\input{sec/3_method}
\input{sec/4_exp}
\input{sec/5_con}

{
    \small
    \bibliographystyle{iclr2027_conference}
    \bibliography{example_paper}
}

\input{sec/6_app}

\end{document}

%% file: sec/1_intro.tex
\section{Introduction}
\vspace{-0.5em}
\rv{Large Vision-Language Models (LVLMs) have achieved remarkable success in static image understanding~\citep{liu2024visual, bai2023qwen} and are rapidly extending to video domains, where models must model temporal dynamics and reason over unfolding events~\citep{zhang2023videollama, li2024llamavid}. Concurrently, Chain-of-Thought (CoT) prompting~\citep{wei2022chain} has emerged as a cornerstone paradigm, encouraging models to generate intermediate reasoning trajectories prior to producing a final answer~\citep{bai2025qwen3vl}~\citep{bai2025qwen3vl}.}

However, applying unstructured, free-form CoT directly to video is inherently {suboptimal}. Unlike static images, video content is characterized by high spatio-temporal redundancy and events distributed across broad temporal intervals. Unstructured CoT frequently yields verbose and diffuse narrations that bury salient temporal boundaries and critical visual cues, causing reasoning-oriented models to underperform even basic instruction-tuned baselines~\citep{bai2025qwen3vl}. In contrast, human perception builds structured mental representations, i.e., first segmenting continuous video into salient entities, actions, and temporal event boundaries, and subsequently performing higher-level cognitive reasoning over these grounded elements~\citep{zacks2007event}.

 Motivated by this perceptual order, recent efforts attempt to ground video reasoning in explicit visual evidence~\citep{meng2025openo3video, huang2026coe, zhang2026eflow}. However, existing methods typically formulate evidence representations for narrow, task-specific applications. To overcome this, we introduce \textbf{Structured Event Evidence}, a unified, task-agnostic intermediate representation. As illustrated in~\cref{fig:teaser,fig:example}, our approach organizes salient entities, actions, objects, scenes, and event boundaries into a compact, temporally ordered schema. This structured evidence anchors the subsequent reasoning process, keeping it concise, interpretable, and verifiable against observed temporal visual cues.   

Despite its advantages, post-training a model to simultaneously extract structured evidence and execute complex reasoning introduces a severe multi-objective optimization challenge. Unlike standard CoT fine-tuning that primarily optimizes for answer accuracy, learning structured event evidence requires balancing three competing post-training objectives: schema compliance, task accuracy, and reasoning efficiency. Standard Group Relative Policy Optimization (GRPO)~\citep{shao2024deepseekmath} collapses these multi-dimensional reward vectors into a scalar via fixed linear weighting, triggering two critical failure modes: \textit{Conflicting Objectives}: The length reward  and task accuracy metrics push optimization in
opposite directions. A weak length penalty causes the model to generate redundant evidence
(e.g., one event per second). Conversely, an aggressive penalty forces the model to be
concise before it learns accuracy, pruning vital temporal evidence and causing performance
to collapse ~\citep{wang2026icr, li2026leash}.
\textit{Indistinguishable Objectives}: Collapsing multi-dimensional rewards into a fixed weighted
sum obscures individual component contributions. Two candidates earning different reward
vectors (e.g., satisfying format and accuracy vs. satisfying format and budget) can yield
identical scalar averages, preventing the group from determining which specific objective
each response met.


To resolve these optimization bottlenecks without manual hyperparameter tuning or expensive per-objective gradient steps, we propose \textbf{\methodfull{} (\method{})}. \method{} formulates post-training as group-relative multi-objective policy optimization, dynamically deriving scale-invariant advantage weights by solving a lightweight minimum-norm problem over centered reward vectors within each rollout group. Furthermore, to prevent optimization instability when learning structure and reasoning simultaneously, \method{} operates on top of a three-stage progressive training curriculum supported by our newly constructed \textbf{PARSE-60K} dataset.

Our primary contributions are summarized as follows:
\vspace{-0.5em}
\begin{itemize}
\setlength{\itemsep}{0pt}
\setlength{\parsep}{0pt}
\item \textbf{Multi-Objective Failure Analysis:} We identify and formalize two core optimization bottlenecks in post-training structured video reasoning, i.e., \textit{objective conflicts} and \textit{scalar advantage degeneracy}, demonstrating why static scalarization in standard GRPO fails.
\item \textbf{The \method{} Algorithm:} We introduce \methodfull{} (\method{}), a scale-invariant multi-objective RL framework that dynamically balances group-relative advantages via minimum-norm quadratic optimization in reward space. \method{} incurs virtually zero computational overhead and is independent of model parameter scale.
\item \textbf{Unified Framework \& Benchmark:} We present \textit{Structured Event Evidence}, a unified schema for grounded video reasoning, along with \textbf{PARSE-60K} (60K annotations over 35K videos) and a three-stage progressive curriculum. Our resulting model, \textbf{PARSE-4B}, achieves state-of-the-art performance across both temporal grounding and complex video question-answering benchmarks.
\end{itemize}


%% file: sec/2_related.tex
\vspace{-1em}
\section{Related Work}
\vspace{-0.5em}
\paragraph{Video understanding LVLMs.}
\rv{Following the success of image-to-text LVLMs~\citep{liu2024visual}, early Video-LLMs largely treated videos as sequences of static frames and relied on token compression to fit spatio-temporal information within a limited context budget~\citep{maaz2023video, cheng2024videollama2}.
While effective for coarse video understanding, this frame-centric formulation can overlook temporal grounding and event-level dependencies~\citep{liu2024etbench, chen2024rextime}. More recent work has introduced time-aware architectures to better capture temporal dynamics~\citep{ren2024timechat, huang2024vtimellm, guotrace}, as well as memory- and tool-augmented mechanisms for reasoning over long, untrimmed videos~\citep{song2024moviechat, wang2024videoagent}.
Nevertheless, many Video-LLMs still depend heavily on sparse frame sampling and selection heuristics rather than explicitly modeling how events evolve and relate over time~\citep{fu2024videomme, xiao2021nextqa}. Our approach instead learns an explicit structured representation of salient events and their temporal dependencies before free-form reasoning, providing a compact event-level prior for downstream reasoning and reducing reliance on brittle frame selection.}

\vspace{-0.5em}
\paragraph{Reasoning and CoT in LVLMs.}
\rv{CoT prompting improves language-model reasoning by eliciting intermediate steps~\citep{wei2022chain}, while Multimodal-CoT extends this paradigm by conditioning such reasoning on visual inputs~\citep{zhang2023multimodal}.
Video presents a particular challenge for unconstrained CoT: visual inputs are highly redundant, while fine-grained temporal cues and event relationships may be sparse and easily overlooked. As a result, free-form visual reasoning can become verbose, drift away from the relevant evidence, or introduce unsupported details when intermediate steps are insufficiently grounded~\citep{chen2023shikra}.
Prior work addresses this in several ways. Tool-augmented approaches, including ReAct-style and MM-ReAct-style methods, introduce external verification or perception steps during reasoning~\citep{yao2023react, yang2023mmreact}, while methods such as Tree-of-Thoughts and self-consistency explore multiple candidate reasoning trajectories~\citep{yao2023tree, wang2022selfconsistency}.
For video reasoning specifically, recent approaches broadly follow two directions. The first lets the model interact with the visual input during inference, for example by invoking tools or revisiting and re-cropping relevant frames~\citep{zeng2026videoo3, ding2025videozoomer, lee2026visioncoach}, at the cost of additional inference-time computation. The second preserves end-to-end reasoning but explicitly grounds it in intermediate evidence, such as timestamped clues, evidence chains, or evidence flows~\citep{meng2025openo3video, huang2026coe, zhang2026eflow}. We follow this second direction. However, prior evidence representations are tailored to a specific task, so a single reward suffices. Our \emph{Structured Event Evidence} compactly encodes salient events, entities, actions, scenes, and temporal relations as a shared intermediate representation for both task families, which constrains subsequent CoT reasoning to stay concise, relevant, and grounded in the video.}

\vspace{-0.5em}
\paragraph{Reinforcement learning in LVLMs.}
\rv{Reinforcement Learning from Human Feedback (RLHF) has become a standard approach for aligning LLMs with human preferences, typically using scalar rewards optimized with PPO~\citep{ouyang2022training, touvron2023llama, schulman2017proximal}. More recently, GRPO~\citep{shao2024deepseekmath} has emerged as an efficient alternative that uses group-wise relative comparisons and avoids explicit value-function estimation, while remaining compatible with rule-based rewards. In multimodal and video settings, GRPO has been applied to tasks including reasoning, temporal grounding, and counting~\citep{feng2025videor1,wang2025timer1}.
Subsequent work improves different aspects of this optimization process. DAPO~\citep{yu2025dapo} introduces dynamic sampling and decoupled clipping to preserve useful gradients under sparse rewards, while GDPO~\citep{liu2026gdpo} normalizes individual reward components so that weaker signals are not overwhelmed by larger ones. Both, however, still combine the components using fixed weights. So when objectives conflict, their disagreement is averaged rather than resolved. This limitation matters in structured video reasoning, where the model must satisfy structural constraints, task performance, and a reasoning budget at the same time.
Multi-objective RL provides a natural framework for such settings, since it seeks solutions that balance competing objectives rather than optimizing a single scalar target~\citep{kong2025emorl, li2025hoe}. Related work on efficient text reasoning treats the reasoning budget as a genuine competing objective rather than a weak regularizer. \citet{wang2026icr} show that a larger length coefficient shortens outputs faster but degrades accuracy disproportionately, giving an unfavourable accuracy--length trade-off. \citet{li2026leash} address this by adapting the length coefficient online through a Lagrangian dual variable, and work on multi-objective alignment reaches a related conclusion that fixed scalarization is limiting~\citep{lu2026dynamicreward}. These approaches, however, adapt one global trade-off over training. We instead solve the weights inside each prompt group with a minimum-norm mix of the centered rewards rather than of policy gradients, which lets the optimization follow the local pattern of agreement and conflict among objectives while keeping task accuracy under a constrained reasoning budget.}

%% file: sec/3_method.tex
\section{Methodology}
\vspace{-0.5em}
\begin{figure}[t]
    \centering
    \begin{minipage}[b]{0.9\linewidth}
        \includegraphics[width=\linewidth]{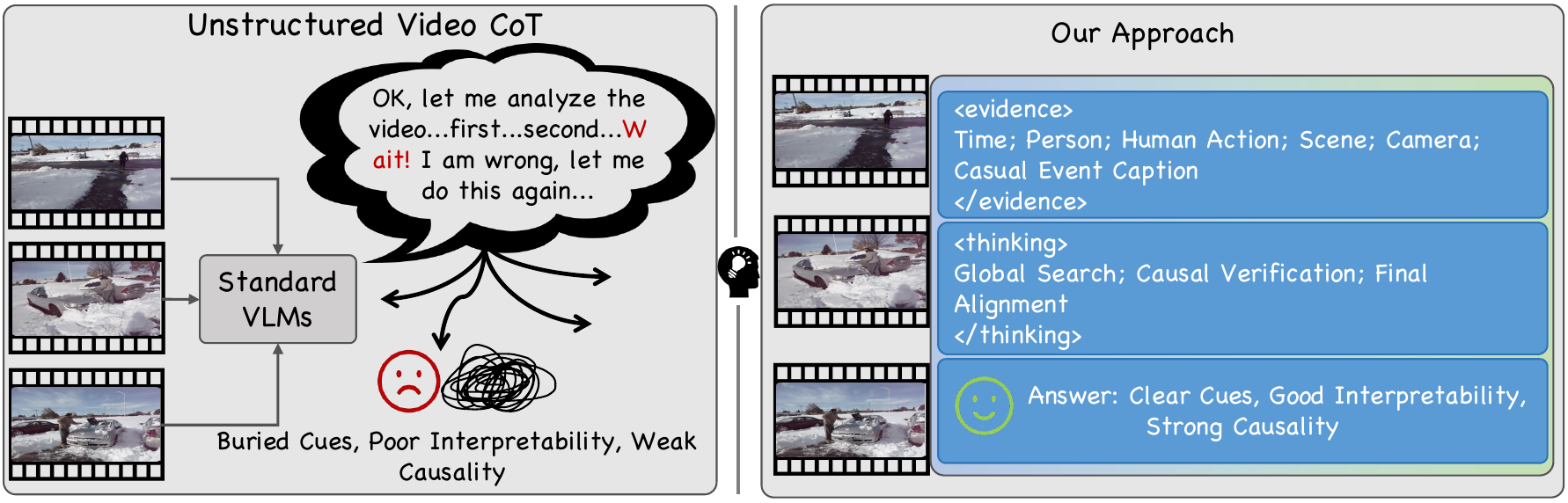}
    \end{minipage}
    \caption{\rv{Comparison of unstructured video CoT and our structured reasoning approach. Our method first extracts structured event evidence and then performs evidence-grounded reasoning over temporal dependencies, yielding clearer cues and more interpretable predictions.} Video from ActivityNet-Captions~\citep{krishna2017dense}.
    }
    \label{fig:example}
\end{figure}

\subsection{Problem Formulation and Reward Space}\label{sec:problem_formulation}
Unlike standard Video-LLMs that rely on unstructured Chain-of-Thought (CoT) and are prone to producing verbose narration, our framework grounds video reasoning in explicit \textit{Structured Event Evidence} to anchor high-level reasoning to concrete temporal boundaries and entity interactions. Formally, given an input video $V$ and a query $q$, the model generates a response sequence $o = (e, z, a)$ containing structured event evidence $e$, a multi-step reasoning path $z$, and a final answer $a$. During post-training, every response $o$ is evaluated across $M$ distinct objectives comprising schema compliance $r_{\text{schema}}$, task accuracy $r_{\text{acc}}$, and reasoning efficiency/budget $r_{\text{len}}$ (\cref{sec:reward_design}). We express the per-sample objective rewards as a vector $\mathbf{r}_i = [r_{i,1}, \dots, r_{i,M}]^\top \in \mathbb{R}^M$.

\subsection{\methodfull{} (\method{})}\label{sec:pfab}

\paragraph{Motivations and Limitations of Static Scalarization.} 
In standard Group Relative Policy Optimization (GRPO)~\citep{shao2024deepseekmath}, the vector reward $\mathbf{r}_i \in \mathbb{R}^M$ is scalarized as $\bar{r}_i = \mathbf{w}^{\top}\mathbf{r}_i$ using a fixed preference vector $\mathbf{w} \in \Delta^{M-1}$, where $\Delta^{M-1}$ denotes the probability simplex. Applying the same weights to every rollout group enforces a global, static trade-off among the $M$ objectives. For joint evidence extraction and reasoning, this static scalarization exhibits two primary optimization failure modes:
\begin{enumerate}
\item \textit{Conflicting Objectives}: The length reward $r_{\text{len}}$ and task accuracy actively compete: insufficient penalties induce redundant evidence over-generation, whereas aggressive penalties prematurely prune critical temporal cues, degrading accuracy~\citep{wang2026icr}.
 \item \textit{Objective Indistinguishability:} Linear scalarization maps non-identical reward vectors $\mathbf{r}_i \neq \mathbf{r}_j$ to identical scalar advantages $\bar{r}_i = \bar{r}_j$ (Figure~\ref{fig:pfab}), collapsing multi-dimensional feedback and obscuring which specific objective a candidate satisfied.
\end{enumerate}
 
While per-objective reward normalization~\citep{liu2026gdpo} mitigates scale imbalances, it leaves directional objective conflicts unresolved. To address these limitations, we introduce \methodfull{} (\method{})  to dynamically calculate scale-invariant, Pareto-optimal advantage weights for each rollout group without computationally expensive per-objective gradient operations. \method{} directly addresses this by dynamically deriving scale-invariant, group-relative advantage weights.

\begin{figure}[t]
\centering
\begin{minipage}[b]{0.95\linewidth}
\includegraphics[width=\linewidth]{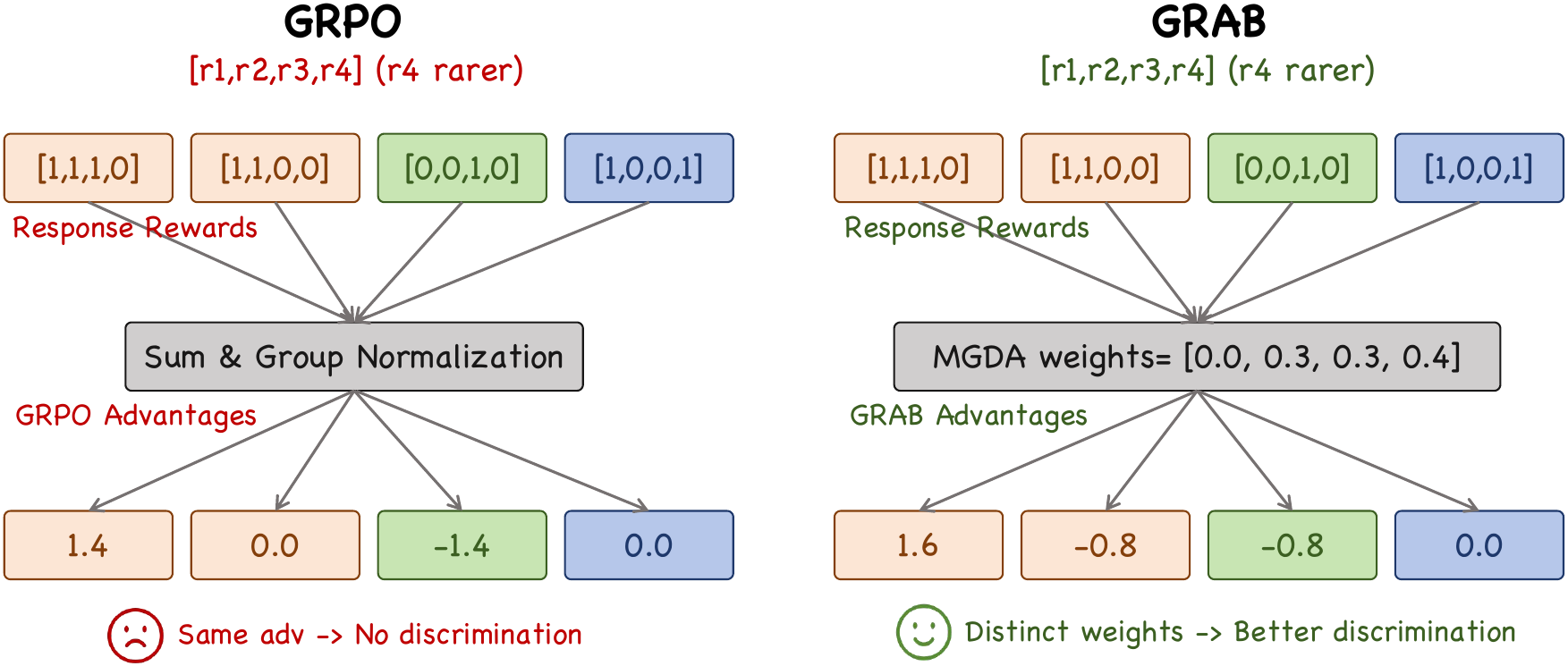}
\end{minipage}
\caption{GRPO vs. \method{} on a rollout group. GRPO averages the four rewards, assigning identical scores and advantages to $[1,1,0,0]$ and $[1,0,0,1]$. \method{} applies dynamic minimum-norm weights (Algorithm~\ref{alg:pfab_groupwise}) to centered rewards before $z$-score normalization, correctly distinguishing response preferences.}
\label{fig:pfab}
\end{figure}

\paragraph{Group-Centered Reward Vectors.} 
For a prompt $q$, we sample a group of $G$ candidate responses $\{o_1, \dots, o_G\} \sim \pi_{\theta_{\text{old}}}(\cdot \mid q)$. Following the group-relative paradigm, we isolate the relative quality of response $o_i$ on objective $m \in \{1, \dots, M\}$ by subtracting the group mean:
\begin{equation}\label{eq:relative_advantage}
\delta_{i,m} = r_{i,m} - \frac{1}{G}\sum_{j=1}^{G} r_{j,m}.
\end{equation}
\noindent where $r_{i,m}$  denotes the reward for response $o_i$ on objective $m$. 
We organize these relative scores into a centered reward matrix $\mathbf{D}_q \in \mathbb{R}^{G \times M}$ where $[\mathbf{D}_q]_{i,m} = \delta_{i,m}$. Each column represents a \textit{directional preference vector} indicating whether responses perform above or below the group average for objective $m$.

\paragraph{Standardized Minimum-Norm Advantage Weighting.}
Inspired by the Multiple Gradient Descent Algorithm (MGDA)~\citep{desideri2012mgda}, \method{} solves for optimal objective weights in a standardized space. To ensure scale invariance, each column of $\mathbf{D}_q$ is divided by its sample standard deviation $\sigma_{q,m}$, yielding the standardized matrix $\hat{\mathbf{D}}_q$. Objectives with variance below a threshold 
$\sigma_{q,m} < \tau_{\sigma}$ are pruned, as they provide uninformative contrast across the group; we denote the set of remaining active objectives as $\mathcal{V}$. We apply the Frank-Wolfe algorithm~\citep{jaggi2013revisitingfrankwolfe} to identify the minimum-norm direction within the simplex $\Delta_{\vert{}\mathcal{V}\vert{}}$:

\begin{align}
\label{eq:mgda_opt_group}
\boldsymbol{\alpha}^{*}_q &= \arg\min_{\boldsymbol{\alpha} \in \Delta_{|\mathcal{V}|}} \left\| \hat{\mathbf{D}}_q \boldsymbol{\alpha} \right\|^{2}, \nonumber \\
\text{where} \quad \Delta_{|\mathcal{V}|} &= \{\boldsymbol{\alpha} \in \mathbb{R}^{|\mathcal{V}|} \mid \textstyle\sum_{m \in \mathcal{V}} \alpha_m = 1,\; \alpha_m \ge 0\}.
\end{align}

Because every column in $\hat{\mathbf{D}}_q$ has norm $\Vert{}\hat{\mathbf{d}}_{q,m}\Vert{}^2 = G$, the Gram matrix simplifies directly to $\hat{\mathbf{D}}_q^\top \hat{\mathbf{D}}_q = G\,\mathbf{C}_q$, where $\mathbf{C}_q \in \mathbb{R}^{\vert{}\mathcal{V}\vert{} \times \vert{}\mathcal{V}\vert{}}$ is the reward correlation matrix across the group. Thus, Eq.~\ref{eq:mgda_opt_group} is exactly equivalent to:
\begin{equation}
\label{eq:alignment}
\boldsymbol{\alpha}^{*}_q \;=\; \arg\min_{\boldsymbol{\alpha} \in \Delta_{|\mathcal{V}|}} \; \boldsymbol{\alpha}^{\top} \mathbf{C}_q \, \boldsymbol{\alpha},
\end{equation}

If the resulting minimum norm is smaller than a threshold $\eta$, we fall back to uniform weights over $\mathcal{V}$. By optimizing over $\mathbf{C}_q$, $\boldsymbol{\alpha}^{*}_q$ dynamically suppresses highly correlated or redundant objectives and assigns higher relative weight to conflicting objectives that exhibit high variance across the rollout group.\paragraph{Advantage Scalarization and Policy Update.}After obtaining $\boldsymbol{\alpha}^{*}_q$, we project these weights back onto the raw centered rewards $\mathbf{D}_q$ to preserve scale fidelity, followed by group-wise $z$-score normalization:

\begin{equation}
\label{eq:final_advantage_group}
A_i^{\text{raw}} = \sum_{m \in \mathcal{V}}\alpha^{*}_{q,m}\,\delta_{i,m}, \quad\quad A_i = \frac{A_i^{\text{raw}} - \mu_{A}}{\sigma_{A} + \epsilon}.
\end{equation}

This two-step formulation—optimizing weights in standardized space for scale-invariance while applying them in centered space for fidelity—allocates relative advantages across objectives while maintaining GRPO's update scale. Finally, $A_i$ is integrated into the clipped surrogate policy objective:
\begin{equation}
\label{eq:policy_loss}
\resizebox{0.75\linewidth}{!}{$
\mathcal{L}(\theta) = -\frac{1}{|\mathcal{B}|} \sum_{i} \left[ \min \left( \rho_i A_i,\; \mathrm{clip}(\rho_i, 1{-}\varepsilon, 1{+}\varepsilon) A_i \right) - \beta \mathbb{D}_{\mathrm{KL}} \right]
$}
\end{equation}

where $\rho_i(\theta) = \frac{\pi_\theta(o_i \mid q)}{\pi_{\theta_{\text{old}}}(o_i \mid q)}$, $\vert{}\mathcal{B}\vert{}$ is the batch size, $\varepsilon$ is the clipping hyperparameter, and $\beta$ controls the KL penalty relative to reference policy $\pi_{\text{ref}}$.

\subsection{Progressive Training Curriculum}\label{sec:training_curriculum}While \method{} resolves multi-objective conflicts during post-training, attempting to optimize structural formatting, temporal grounding, and multi-step reasoning simultaneously from a raw base model leads to severe convergence instability. We therefore decompose training into a three-stage progressive curriculum:
\begin{enumerate}
\item \textbf{Evidence Training (Stage 1):} The model is trained exclusively to generate high-quality structured evidence, learning to ground salient entities, actions, scenes, objects, and temporal event boundaries using supervised cross-entropy loss.
\item \textbf{Thinking Warm-Start (Stage 2):} Conditioned on the extracted structured evidence, the model learns to execute multi-step reasoning through task-oriented supervision across diverse temporal grounding and visual reasoning tasks (detailed prompts in Appendix~\cref{appendix:data_curation_prompts}).
\item \textbf{RL Post-Training (Stage 3):} Stage 3 applies \method{} during RL post-training to dynamically resolve the competing objectives of joint evidence extraction, reasoning quality, and concise generation.
\end{enumerate}

\subsection{PARSE-60K Dataset Engine}
\label{sec:dataset_engine}
To provide the tailored supervision required across all three stages, we construct \textbf{PARSE-60K}, comprising 60K annotations over 35K clips sampled from Panda-70M~\citep{chen2024panda70m}. We exclude Charades and ActivityNet from training, and remove any video whose ID occurs in a reported evaluation set. The dataset supports all three stages on a shared set of videos: Stage 1 provides event-boundary and entity annotations; Stage 2 provides step-by-step reasoning traces; and Stage 3 provides task-oriented prompts. \Cref{fig:datapip} summarises the construction pipeline; prompts, scoring rubrics, and additional implementation details are provided in \Cref{appendix:data_pipeline}.

Because Panda-70M contains a single caption per video but no event timestamps, we first apply shot detection and merge or split the resulting segments to obtain clips of $10$--$30$ seconds. Humans audit a randomly sampled half of the resulting boundaries. We generate evidence annotations using
Qwen3-VL-235B-Instruct~\citep{bai2025qwen3vl} and Gemini-2.5-Pro~\citep{comanici2025gemini-2-5}, which alternately serve as generator and judge.
Human review covers $30$K of the $60$K annotations (three annotators; mean scores $8.5$, $7.9$, and $8.2$). If any sample receives a score below $6$, its entire batch of $1{,}000$  is regenerated; this criterion triggered regeneration for 
$28.3\%$ of batches. 
The RL data comprise
$53\%$ temporal grounding, $21\%$ spatial VQA, $20\%$ reasoning VQA, and $6\%$ temporal VQA.

\begin{figure*}[t]
    \centering
    \includegraphics[width=\textwidth]{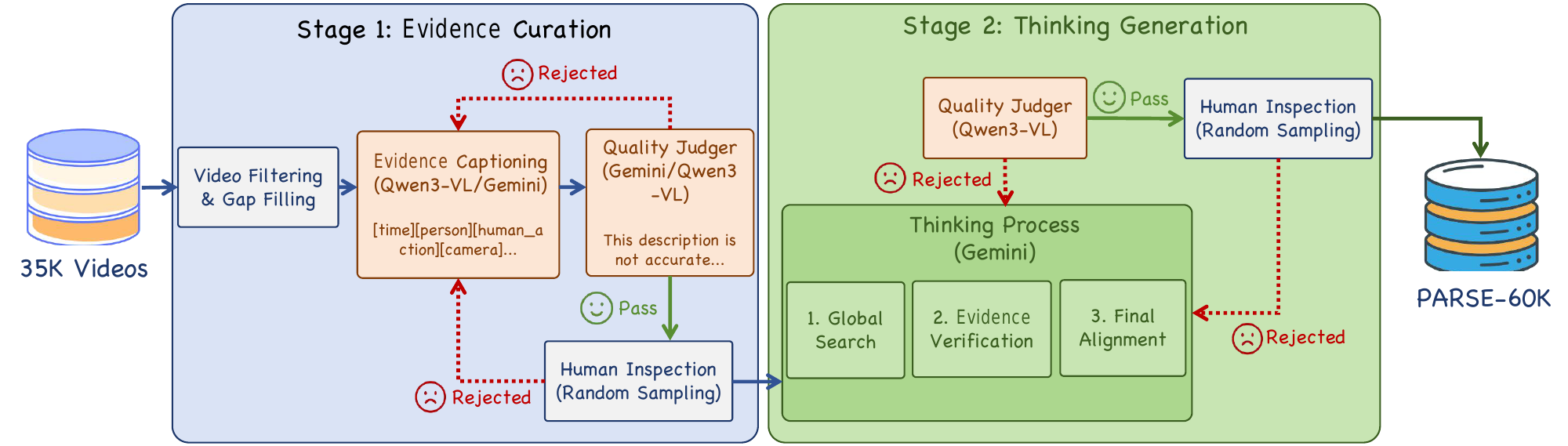}
    \caption{How we build PARSE-60K. We first drop any video whose ID appears in an eval set. Stage 1 then cuts events, fills gaps, writes evidence, and judges quality. Stage 2 writes thinking traces and judges them again.}
    \vspace{-0.5em}
    \label{fig:datapip}
\end{figure*}

%% file: sec/4_exp.tex
\vspace{-0.5em}
\section{Experiments}
\vspace{-0.5em}
We use \textit{Qwen3-VL-4B-Instruct} as our base model with a three-stage training pipeline. Stages 1--2 use LoRA fine-tuning; Stage 3 (RL) performs full-parameter updates with four reward objectives (\textit{Format, Linear IoU, Multi-choice Accuracy, Length}). We set $\mathrm{fps}=1$ and evaluate all Qwen3 baselines under the same setting; other baselines use official results (typically $\mathrm{fps}=2$). \rv{Every number we produce ourselves is a three-seed mean, and Appendix~\ref{sec:seed_variance} reports the spread.} \textbf{\textit{Additional analyses including TempCompass/EventBench evaluation, 8B scalability, \method{} overhead, data quality, the scope of our claims, and reasoning quality are provided in Appendices.}}

\vspace{-0.5em}
\subsection{Ablation Studies}
\vspace{-0.5em}
\label{sec:ablation}
\textbf{Ablation setup.} We evaluate component effectiveness on Charades-TimeLens~\citep{zhang2025timelens} and ActivityNet-Captions~\citep{krishna2017dense} for temporal grounding, and VideoMME~\citep{fu2024videomme} and MLVU~\citep{zhou2025mlvu} for general understanding. Table~\ref{tab:ablation} reports component ablations and policy optimization baselines. The component ablations remove the evidence block (\textit{w/o Evidence}), structured thinking (\textit{w/o Thinking}), or RL (\textit{w/o RL}). The optimization baselines include \textit{GRPO}, \textit{DAPO}~\citep{yu2025dapo}, \textit{GDPO}~\citep{liu2026gdpo}, \textit{Fixed Weight (Accuracy)}, and \textit{Lagrangian (Length)}. We evaluate the policy optimization methods with group sizes $G=8$ and $G=16$. All policy optimization variants start from the same Stage-2 checkpoint and use the same RL corpus, so they differ only in the optimization algorithm. For \textit{Fixed Weight (Accuracy)}, the reward weights are $(w_{\text{fmt}}, w_{\text{task}}, w_{\text{len}})=(1,2,1)$. For \textit{Lagrangian (Length)}, the format and task weights remain fixed while one global coefficient on $r_{\text{len}}$ is adjusted online. We tune its budget so that its final output length matches \method{}'s. The update rule is given in Appendix~\ref{sec:lagrangian_baseline}. Dataset details appear in Appendix~\ref{sec:benchmark_details}.

\textbf{Component ablations.} As shown in Table~\ref{tab:ablation}, the three stages serve different roles. Removing structured thinking (\textit{w/o Thinking}) lowers both grounding benchmarks below the untrained backbone. This shows that the evidence becomes useful only when the model reasons over it. Without RL, the model retains the full evidence and thinking supervision and remains above the backbone on grounding, but its outputs are much longer than those of RL-trained models. Removing the evidence block (\textit{w/o Evidence}) keeps the thinking supervision and RL, but loses to \textit{w/o RL} on every accuracy column. RL alone therefore does not recover the benefit of the schema. Stages 1 and 2 teach the format, while Stage 3 trades accuracy against the reasoning budget.
\input{tables/abaltions}

\rv{\textbf{Per-group weights versus one global coefficient.} Lagrangian (Length) keeps the format-to-task ratio of Fixed Weight (Accuracy) and adapts only the length coefficient, using one global value for the entire run (Appendix~\ref{sec:lagrangian_baseline}). We tune its budget so that it ends at the same output length as \method{}, which leaves accuracy as the only axis of comparison. Under this matched budget, \method{} is more accurate on every accuracy column in Table~\ref{tab:ablation}. Reaching a short policy is therefore not the difficulty; retaining accuracy while getting there is. An online dual variable still commits to a single exchange rate at any point in training, but the appropriate rate differs across prompt groups, because a group whose rollouts already agree can be shortened while a group whose rollouts still disagree cannot. GDPO falls short for the same reason, since it normalizes each reward but continues to mix them with fixed weights.}

\rv{\textbf{Fixed weights at a smaller budget.} Fixed Weight (Accuracy) is the strongest static configuration we found, and it is the comparison that matters most for a multi-objective claim. \method{} matches or exceeds it on grounding at every IoU threshold while producing substantially shorter responses, and the two methods trade places on the remaining general-understanding columns of Table~\ref{tab:ablation}. The insight is not that one operating point dominates the other, but that \method{} reaches a competitive point without anyone selecting it. A static weighting commits to one exchange rate between task performance and the reasoning budget before training begins~\citep{wang2026icr, li2026leash}, and the sweep in Appendix~\ref{sec:weight_sweep} shows that this choice moves both accuracy and length considerably. Every candidate rate costs a separate run, whereas \method{} resolves the weights inside each group and leaves no rate to tune.}

\rv{\textbf{Uniform-weight optimizer baselines.} GRPO, DAPO, and GDPO use the uniform reward weights in \cref{tab:reward_hparams}. Against equal weighting, \method{} leads on every accuracy column at both group sizes and is also the most concise of the four methods, so the gain does not come from spending more tokens. Its margin over GRPO persists at both group sizes, which indicates that larger groups alone do not resolve the conflict between objectives. Averaging the reward components discards information that is present in the group, and recovering that information requires weighting the components rather than sampling more responses.}

\input{tables/vtg}
\input{tables/etbench}
\vspace{-0.5em}
\subsection{Comparisons with Existing Video-LLMs}
\label{sec:comparisons}
\vspace{-0.5em}
We compare PARSE-4B against leading Video-LLMs on temporal grounding and general video understanding in Table~\ref{tab:vtg} and Table~\ref{tab:bench}, respectively. To ensure fair comparison, we evaluate key open-source baselines under \rv{both the 2 fps and 1 fps settings}. Beyond the quantitative gains, we highlight three qualitative insights that emerge from the results.

\textbf{PARSE-4B improves over its own 4B backbone.} \rv{Under identical evaluation settings, PARSE-4B improves over Qwen3-VL-4B-Instruct on every benchmark in Table~\ref{tab:vtg} and on six of the eight ETBench sub-tasks in Table~\ref{tab:bench}.} Table~\ref{tab:ablation} shows this gap is the full pipeline, not the schema alone: Stage 1--2 help a little on grounding, and Stage 3 RL does most of the work. \rv{This same-base comparison rules out that the gains stem from the backbone alone.} We interpret the structured output as an \textit{information bottleneck} complementary to the base model: the backbone supplies fine-grained perception, while the evidence schema imposes a discrete, time-ordered abstraction so that reasoning operates on compact event-level representations rather than thousands of redundant frame tokens.

\textbf{\rv{Explicit event timestamps improve localization.}} A consistent pattern in Table~\ref{tab:vtg} is that PARSE-4B improves at every IoU threshold over the same backbone. The explicit event timestamps encoded in the evidence provide hard anchors for the start and end of activities, enabling the evidence verification stage to align predictions with ground-truth intervals. \rv{The absolute gain is larger when the reference annotation is looser.} ActivityNet-Captions has broad boundaries and improves the most, while ActivityNet-TimeLens re-annotates the same videos with tighter ones and improves less. The model finds the right region on both, and it refines the boundary itself mainly when the reference granularity is close to the event granularity our schema assumes.

\textbf{Unstructured thinking is counterproductive for video.} Across nearly all benchmarks in both tables, \textit{Qwen3-VL-4B-Thinking} underperforms its \textit{Instruct} counterpart, a trend also observed by~\citep{bai2025qwen3vl}. The token budget is spent in the wrong place: unstructured CoT spends tokens on verbose self-corrections, repetitive scene narrations, and speculative tangents, leaving insufficient capacity for the actual temporal reasoning that matters. Front-loading a compact evidence block leaves downstream thinking tokens for evidence-grounded inference rather than redundant perception. \rv{On NExT-GQA, PARSE-4B reaches $73.6$ against $72.1$ for Instruct and $66.6$ for Thinking, the widest gap to the Thinking baseline in either table, on a benchmark built around Why and How questions.}

%% file: tables/abaltions.tex
\begin{table*}[t]
    \centering
    \setlength{\tabcolsep}{4pt}
    \resizebox{1.0\linewidth}{!}{%
    \begin{tabular}{@{}lccccccccccccc@{}}
    \toprule
    \multirow{2}{*}{\textit{\textbf{Model}}} & \multicolumn{4}{c}{\textbf{Charades-TimeLens} (R1@IoU)} & \multicolumn{4}{c}{\textbf{ActivityNet-Captions} (R1@IoU)} & \multicolumn{2}{c}{\textbf{VideoMME}} & \multicolumn{3}{c}{\textbf{MLVU}} \\ \cmidrule(l){2-14} 
                                & @0.3          & @0.5          & @0.7          & \multicolumn{1}{c|}{Len}           & @0.3          & @0.5          & @0.7          & \multicolumn{1}{c|}{Len}            & Acc           & \multicolumn{1}{c|}{Len}           & TR            & Ego           & Len          \\ \midrule
    \textit{Baselines~\color{gray}{1fps}}   &               &               &               & \multicolumn{1}{c|}{}              &               &               &               & \multicolumn{1}{c|}{}               &               & \multicolumn{1}{c|}{}              &               &               &              \\
    Qwen3-VL-4B-Inst.*~\cite{bai2025qwen3vl}  & 56.5          & 37.8          & 18.4          & \multicolumn{1}{c|}{-}             & 50.2          & 35.8          & 21.6          & \multicolumn{1}{c|}{-}              & 63.4          & \multicolumn{1}{c|}{-}             & 80.1          & 60.3          & -            \\
    Qwen3-VL-4B-Think.*~\cite{bai2025qwen3vl} & 56.3          & 37.4          & 17.8          & \multicolumn{1}{c|}{1439} & 47.8          & 31.7          & 18.8          & \multicolumn{1}{c|}{1895} & 63.1          & \multicolumn{1}{c|}{4016} & 79.5          & 59.0          & 5192            \\ \midrule
    \textit{Ablations~\color{gray}{1fps}}   &               &               &               & \multicolumn{1}{c|}{}              &               &               &               & \multicolumn{1}{c|}{}               &               & \multicolumn{1}{c|}{}              &               &               &              \\
    w/o Evidence = 16                       & 54.0          & 37.0          & 16.5          & \multicolumn{1}{c|}{553} & 50.7          & 34.3          & 23.2          & \multicolumn{1}{c|}{1158} & 60.8          & \multicolumn{1}{c|}{2961} & 76.2          & 56.8          & 3128            \\
    w/o Thinking = 16                       & 50.1          & 34.2          & 14.1          & \multicolumn{1}{c|}{673} & 45.5          & 31.2          & 17.8          & \multicolumn{1}{c|}{2884} & 57.5          & \multicolumn{1}{c|}{8432} & 75.6          & 55.4          & 10118            \\
    w/o RL                                  & 57.0          & 38.2          & 18.8          & \multicolumn{1}{c|}{1174} & 53.4          & 36.6          & 23.7          & \multicolumn{1}{c|}{3386} & 62.5          & \multicolumn{1}{c|}{11258} & 81.2          & 61.3          & 13187            \\
    GRPO = 8                                & 57.2          & 38.3          & 19.4          & \multicolumn{1}{c|}{842}           & 55.8          & 37.9          & 23.6          & \multicolumn{1}{c|}{1955}           & 60.2          & \multicolumn{1}{c|}{4179}          & 81.5          & 61.9          & 5635            \\
DAPO = 8                     & 57.4 & 38.7 & 19.1 & \multicolumn{1}{c|}{816}           & 56.0 & 37.3 & 24.0 & \multicolumn{1}{c|}{1829}           & 61.7 & \multicolumn{1}{c|}{4108} & 80.1 & 61.4 & 5540            \\
GDPO = 8                     & 57.8 & 39.6 & 20.4 & \multicolumn{1}{c|}{734}           & 56.2 & 38.0 & 24.6 & \multicolumn{1}{c|}{1703}           & 62.6 & \multicolumn{1}{c|}{3987} & 80.8 & 61.8 & 5376            \\
    \method{} = 8                               & 59.1          & 41.1          & 21.6          & \multicolumn{1}{c|}{\textbf{503}}           & 56.7          & 40.2          & 25.1          & \multicolumn{1}{c|}{\textbf{1116}}           & 64.7          & \multicolumn{1}{c|}{\textbf{3558}}          & 82.6          & 63.3          & 4798            \\
GRPO = 16                    & 59.8 & 40.6 & 21.8 & \multicolumn{1}{c|}{846}           & 60.2 & 41.0 & 24.7 & \multicolumn{1}{c|}{2038}           & 65.6 & \multicolumn{1}{c|}{4388} & 83.1 & 63.8 & 5327            \\
DAPO = 16                    & 61.2 & 41.3 & 21.2 & \multicolumn{1}{c|}{714}           & 61.2 & 41.3 & 26.1 & \multicolumn{1}{c|}{1921}           & 65.1 & \multicolumn{1}{c|}{4313} & 83.6 & 64.1 & 5238            \\
GDPO = 16                    & 61.2 & 41.6 & 22.2 & \multicolumn{1}{c|}{770}           & 62.2 & 41.0 & 26.7 & \multicolumn{1}{c|}{1722}           & 65.6 & \multicolumn{1}{c|}{4186} & 83.9 & 64.8 & 5083            \\
Fixed Weight (Accuracy) = 16     & 63.2 & 42.9 & 24.5 & \multicolumn{1}{c|}{931} & 63.9 & 44.0 & 28.0 & \multicolumn{1}{c|}{2013} & \textbf{68.0} & \multicolumn{1}{c|}{4310} & 86.1 & \textbf{67.2} & 5310            \\
Lagrangian (Length) = 16    & 62.1 & 40.2 & 23.5 & \multicolumn{1}{c|}{556} & 62.1 & 42.2 & 26.6 & \multicolumn{1}{c|}{1212} & 65.7 & \multicolumn{1}{c|}{3806} & 86.1 & 64.4 & \textbf{4489}            \\
\textbf{\method{} = 16}          & \textbf{63.7} & \textbf{43.4} & \textbf{24.6} & \multicolumn{1}{c|}{579}  & \textbf{64.7} & \textbf{44.2} & \textbf{28.1} & \multicolumn{1}{c|}{1231}  & 67.7 & \multicolumn{1}{c|}{3736} & \textbf{86.9} & 66.4 & 4536            \\
    \bottomrule
    \end{tabular}
    }
    \caption{Ablation study at the 4B scale on temporal grounding and general video understanding. We remove evidence, structured thinking, or RL, and compare \method{} against GRPO, DAPO, GDPO, fixed reward weighting, and a Lagrangian length baseline. @$\tau$ is Recall@1 at IoU threshold $\tau$, VideoMME accuracy excludes subtitles, and TR and Ego are the Topic Reasoning and Ego Reasoning sub-tasks of MLVU. Len is the average number of generated tokens per response, so lower is better at comparable accuracy. \rv{A row ending in $=8$ or $=16$ gives the group size of its RL stage, and \textit{w/o RL} has none. Entries marked ``-'' are not reported, since Qwen3-VL-4B-Instruct returns a bare answer whose length is not comparable with the reasoning rows.}}
    \label{tab:ablation}
    \end{table*}

%% file: tables/vtg.tex
\begin{table*}[t]
\centering
\resizebox{0.95\linewidth}{!}{%
\begin{tabular}{@{}lccc|ccc|ccc@{}}
\toprule
\multirow{2}{*}{\textit{\textbf{Model}}}        & \multicolumn{3}{c|}{\textbf{Charades-TimeLens}} & \multicolumn{3}{c|}{\textbf{ActivityNet-TimeLens}} & \multicolumn{3}{c}{\textbf{ActivityNet-Captions}}          \\ \cmidrule(l){2-10} 
                                                & R1@0.3         & R1@0.5         & R1@0.7        & R1@0.3          & R1@0.5          & R1@0.7         & \multicolumn{1}{l}{R1@0.3} & R1@0.5        & R1@0.7        \\ \midrule
\textit{Close-Source Models~\color{gray}{2fps}} &                &                &               &                 &                 &                & \multicolumn{1}{l}{}       &               &               \\
GPT-4o~\citep{hurst2024gpt-4o}                   & 60.6           & 44.5           & 23.5          & 55.2            & 41.4            & 25.8           & -                          & -             & -             \\
GPT-5~\citep{OpenAI2025_GPT5}                    & 59.3           & 42.0           & 22.0          & 57.4            & 44.9            & 30.4           & -                          & -             & -             \\
Gemini-2.0-Flash~\citep{comanici2025gemini-2-5}  & 66.4           & 53.5           & 27.1          & 62.9            & 54.0            & 37.7           & 50.4                       & 33.2          & 19.9          \\
Gemini-2.5-Flash~\citep{comanici2025gemini-2-5}  & 68.7           & 56.1           & 30.6          & 66.8            & 57.5            & 41.3           & 51.2                       & 34.7          & 21.0          \\
Gemini-2.5-Pro~\citep{comanici2025gemini-2-5}    & 74.1           & 61.1           & 34.0          & 72.3            & 64.2            & 47.1           & -                          & -             & -             \\ \midrule
\textit{Open-Source Models~\color{gray}{2fps}}  &                &                &               &                 &                 &                & \multicolumn{1}{l}{}       &               &               \\
VideoChat-R1-7B~\citep{li2025videochat_r1}       & 51.9           & 30.8           & 11.7          & 35.0            & 23.9            & 11.3           & -                          & 33.3          & 16.7          \\
TRACE-7B~\citep{guotrace}                                        & -              & -              & -             & -               & -               & -              & 54.0                       & 37.7          & 24.0          \\
Time-R1-7B~\citep{wang2025timer1}                & 57.9           & 32.0           & 16.9          & 44.8            & 31.0            & 19.0           & 58.1                       & 39.0          & 21.4          \\
Qwen2.5-VL-7B~\citep{qwen2-5-vl}                 & 59.7           & 37.8           & 16.6          & 44.1            & 31.0            & 16.1           & 34.5                       & 20.8          & 11.2          \\ \midrule
\textit{Open-Source Models~\color{gray}{1fps}}  &                &                &               &                 &                 &                & \multicolumn{1}{l}{}       &               &               \\
VideoChat-R1-7B~\citep{li2025videochat_r1}       & 48.7           & 27.1           & 9.7           & -               & -               & -              & -                          & -             & -             \\
Time-R1-7B~\citep{wang2025timer1}                & 55.9           & 30.4           & 13.5          & -               & -               & -              & 56.9                       & 37.2          & 20.1          \\
TRACE-7B~\citep{guotrace}                        & -              & -              & -             & -               & -               & -              & 52.7                       & 35.9          & 22.8          \\
Qwen3-VL-4B-Instruct*~\citep{bai2025qwen3vl}     & 56.5           & 37.8           & 18.4          & 50.2            & 36.2            & 25.2           & 50.2                       & 35.8          & 21.6          \\
Qwen3-VL-4B-Thinking*~\citep{bai2025qwen3vl}     & 56.3           & 37.4           & 17.8          & 48.2            & 36.7            & 24.9           & 47.8                       & 31.7          & 18.8          \\
\textbf{PARSE-4B}                              & \textbf{63.7}  & \textbf{43.4}  & \textbf{24.6} & \textbf{54.7}   & \textbf{41.4}   & \textbf{26.7}  & \textbf{64.7}              & \textbf{44.2} & \textbf{28.1} \\ \bottomrule
\end{tabular}
}
\caption{Experiments on the public video temporal understanding benchmarks. * means this model is reproduced under the same video settings of our model, otherwise, the results are from official reports with more video frames.}
\label{tab:vtg}
\end{table*}

%% file: tables/etbench.tex
\begin{table*}[ht]
\centering
\resizebox{1.0\linewidth}{!}{
\begin{tabular}{@{}lcccccccc|c|cc|c@{}}
\toprule
                                                & \multicolumn{8}{c|}{\textbf{ETBench}}                                                                                         & \multicolumn{1}{l|}{\textbf{VideoMME}} & \multicolumn{2}{c|}{\textbf{MLVU}}                                                                                                         & \multicolumn{1}{l}{\textbf{NExT-GQA}} \\ \midrule
\multirow{2}{*}{\textit{\textbf{Metrics}}}      & RAR           & ECA           & RVQ           & TVG           & EPM           & TAL           & TEM           & GVQ           & \multirow{2}{*}{Acc}                   & \multirow{2}{*}{\begin{tabular}[c]{@{}c@{}}TR\\ (Acc)\end{tabular}} & \multirow{2}{*}{\begin{tabular}[c]{@{}c@{}}Ego\\ (Acc)\end{tabular}} & \multirow{2}{*}{Acc}                  \\
                                                & Acc           & Acc           & Acc           & F1            & F1            & F1            & Rec           & Acc           &                                        &                                                                     &                                                                      &                                       \\ \midrule
\textit{Close-Source Models~\color{gray}{2fps}} &               &               &               &               &               &               &               &               & \multicolumn{1}{l|}{}                  &                                                                     &                                                                      & \multicolumn{1}{l}{}                  \\
GPT-4o~\citep{hurst2024gpt-4o}                   & 27.8          & 27.3          & 57.7          & 40.4          & 4.5           & 20.0          & 13.6          & -             & 71.9                                   & -                                                                   & -                                                                    & -                                     \\
GPT-4v~\citep{openai2023gpt4v}                   & 33.3          & 40.9          & 46.2          & 27.0          & 1.8           & 18.0          & 23.9          & -             & 59.9                                   & -                                                                   & -                                                                    & -                                     \\ \midrule
\textit{Open-Source Models}  &               &               &               &               &               &               &               &               & \multicolumn{1}{l|}{}                  &                                                                     &                                                                      & \multicolumn{1}{l}{}                  \\
VideoLLama2-7B~\citep{cheng2024videollama2}      & 28.8          & 27.4          & 28.0          & 0.1           & 0.0           & 0.0           & 0.0           & -             & -                                      & -                                                                   & -                                                                    & -                                     \\
VTimeLLM-7B~\citep{huang2024vtimellm}            & 28.4          & 31.0          & 28.0          & 7.6           & 1.9           & 18.2          & 6.9           & 1.9           & -                                      & -                                                                   & -                                                                    & 17.4                                  \\
TRACE-7B~\citep{guotrace}                        & 29.4          & 28.8          & 42.6          & 46.8          & 12.3          & 21.6          & 17.8          & 52.4          & 49.6                                   & -                                                                   & -                                                                    & -                                     \\
TinyLLaVA-Video-R1-3B~\citep{zhang2025tinyllava} & -             & -             & -             & -             & -             & -             & -             & -             & 46.6                                   & -                                                                   & -                                                                    & -                                     \\
Open-o3-Video~\citep{meng2025openo3video} & -  & -  & -  & -  & -  & -  & -  & -  & 63.6 & -  & -  & -  \\
Qwen2.5-VL-7B~\citep{qwen2-5-vl}                 & 41.0          & 42.0          & 51.4          & 20.1          & 2.9           & 15.3          & 22.3          & 32.7          & 62.9                                   & 74.9                                                                   & 50.0                                                                    & 59.5                                  \\
Qwen2.5-VL-3B~\citep{qwen2-5-vl}                 & 38.4          & 47.2          & 52.2          & 19.9          & 2.6           & 14.8          & 18.1          & 55.1          & 61.5                                   & -                                                                   & -                                                                    & -                                     \\
\midrule
\textit{Reproduced by us~\color{gray}{1fps}}  &               &               &               &               &               &               &               &               & \multicolumn{1}{l|}{}                  &                                                                     &                                                                      & \multicolumn{1}{l}{}                  \\
Qwen2.5-VL-3B + \method{}*          & 42.7 & 50.4 & 54.1 & 32.0 & 10.1 & 16.9 & 21.3 & 55.7 & 64.0 & -  & -  & -  \\
Qwen3-VL-4B-Instruct*~\citep{bai2025qwen3vl}      & 53.1          & \textbf{54.1} & 54.3          & 64.6          & 14.8          & \textbf{20.1} & 9.6           & 57.7          & 63.4                                   & 80.1                                                                & 60.3                                                                 & 72.1                                  \\
Qwen3-VL-4B-Thinking*~\citep{bai2025qwen3vl}      & 53.0          & 51.0          & 54.2          & 64.3          & 15.7          & 20.0          & 20.4          & 55.6          & 63.1                                   & 79.5                                                                & 59.0                                                                 & 66.6                                  \\
\textbf{PARSE-4B}                              & \textbf{56.1} & 53.6          & \textbf{57.2} & \textbf{67.3} & \textbf{16.9} & 19.9          & \textbf{27.9} & \textbf{59.6} & \textbf{67.7}                          & \textbf{86.9}                                                       & \textbf{66.4}                                                        & \textbf{73.6}                         \\ \bottomrule

\end{tabular}
}
\caption{Experiments on the public general video understanding benchmarks. TR and Ego mean Topic Reasoning and Ego Reasoning, respectively. \rv{Rows marked * are run by us at $1$ fps, under the same video settings as PARSE-4B. The other rows are taken from official reports, which usually use more video frames.}}
\label{tab:bench}
\end{table*}

%% file: sec/5_con.tex
\vspace{-1em}
\section{Conclusion}
\vspace{-0.5em}
\rv{We present \textit{PARSE-4B}, trained with \method{}, which weights the rewards inside each prompt group instead of fixing one weight vector for the whole run. \method{} arrives at a working point nobody selected: it improves on seven of nine accuracy columns over hand-tuned static weights at $13$ to $39\%$ fewer tokens, and it is more accurate on all nine than an online global length coefficient tuned to the same output length. Which member of the static family to use is a question \method{} does not ask. We instantiate the setting on \textit{Structured Event Evidence} as the output format that creates the conflict between the schema, the task metric and the reasoning budget, and curate \textit{PARSE-60K} to teach that format before RL. \method{} reads only the reward matrix of a group, so it carries over to any setting where several rewards score the same rollouts.}

%% file: sec/6_app.tex
%
%
\newpage
\appendix
\onecolumn
\startcontents[app]
{
    \hypersetup{linkcolor=black}
    \parskip=0em
    \renewcommand{\contentsname}{Contents of Appendix}
    \printcontents[app]{ }{1}{\section*{\contentsname}}
}

\section{Implementation Details}
\label{sec:impl_details}

We use \textit{Qwen3-VL-4B-Instruct} as our base model. Since the 4B model is more susceptible to general-capability forgetting, we adopt stage-specific hyperparameter settings across our training stages. For Stage 1, we fine-tune with LoRA (rank $64$, $\alpha=128$), a learning rate of $1\times10^{-5}$, and weight decay of $1\times10^{-4}$. In Stage 2, we keep the same LoRA~\citep{hu2022lora} configuration but increase the learning rate to $1\times10^{-4}$ and weight decay to $2\times10^{-4}$.

For RL Stage 3, we perform full-parameter updates with a learning rate of $5\times10^{-6}$, $\mathrm{num\_generation}=16$ for the final PARSE-4B model, and video max tokens per frame set to $64$. The ablation study in Table~\ref{tab:ablation} additionally reports all policy optimization variants under $\mathrm{num\_generation}=8$. We set $\mathrm{video\_fps}=1.0$ and $\mathrm{epoch}=1.0$ across all stages, with $\mathrm{video\_max\_frames}=256$ for Stages 1--2 and $\mathrm{video\_max\_frames}=128$ for Stage 3 to save memory. During evaluation, we re-evaluate all Qwen3 series models under the same settings as ours ($\mathrm{fps}=1$, $\mathrm{max\_frames}=256$). For other models, results are taken from official reports, typically with $\mathrm{fps}=2$ and $\mathrm{max\_frames}=2048$. \rv{Every number we produce ourselves is the mean of three random seeds, and \cref{sec:seed_variance} reports the spread across them on Charades-TimeLens.}

\rv{\cref{tab:reward_hparams} collects the reward and solver hyperparameters of Stage 3, including the ones the length reward and the static baselines depend on.}

\input{tables/reward_hparams}

\section{\rv{Limitations}}
\label{sec:limitations}

Our training corpus remains
limited in scale, \rv{so} expanding annotation volume and diversity is an
important next step. \rv{Our main experiments also use a
4B-parameter model because of computational constraints, and scaling to larger
architectures is future work that depends on additional compute. Two further limits come from the schema itself. First, we cut events into $10$--$30$ second segments, which puts a floor on how precisely the model can place a boundary, and an adaptive or multi-scale decomposition of events would be the way to lift it. Second, the six fields are built around people and their actions, and a rule ties the action field to the person field, so on video without people, such as wildlife footage or screen recordings, several fields would fall back to ``None'' and the schema would carry less information.}

\section{Detailed Information of Evaluation Benchmarks}
\label{sec:benchmark_details}

To comprehensively evaluate the capabilities of our model, particularly in temporal reasoning and open-ended generation, we utilize a diverse set of benchmarks categorized into two primary domains: Video Temporal Grounding and General Video Understanding.

\subsection{Video Temporal Grounding Benchmarks}
These benchmarks evaluate the model's ability to precisely localize specific moments within untrimmed videos based on textual queries.

\begin{itemize}
    \item \textbf{Charades-TimeLens}~\citep{zhang2025timelens}: A high-quality dataset re-annotated based on the Charades-STA dataset, where 66.3\% of low-quality annotations were rewritten. The final dataset contains 1,313 videos and 3,363 annotations, with an average video duration of 29.3 seconds. This benchmark focuses on daily indoor activities (e.g., cooking, reading) and is characterized by dense actions and complex temporal overlaps. Unlike the noisy annotations in the original dataset, Charades-TimeLens provides rigorously verified timestamps to evaluate fine-grained temporal localization. We report Recall@1 at IoU thresholds of 0.3, 0.5, and 0.7.

    \item \textbf{ActivityNet-TimeLens}~\citep{zhang2025timelens}: A re-annotated version of the ActivityNet-Captions~\citep{krishna2017dense} dataset introduced in the TimeLens benchmark, where 69.7\% of low-quality annotations were rewritten. The final dataset contains 1455 videos and 4500 annotations, with an average video duration of 134.9s. It addresses the issue of loose boundaries in the original annotations by providing stricter and more precise start/end times for temporal queries. Evaluating on this split offers a more rigorous assessment of a model's boundary alignment capabilities.
    
    \item \textbf{ActivityNet-Captions}~\citep{krishna2017dense}: A large-scale Video Temporal Grounding benchmark containing 4885 open-domain YouTube videos for evaluation, with 17031 annotations and an average video duration of 118s. It covers a wide range of complex human activities. The dataset is widely used for dense video captioning and temporal grounding. We evaluate performance using Recall@1 at IoU thresholds of 0.3, 0.5, and 0.7.
    
\end{itemize}

\subsection{General Video Understanding Benchmarks}
These benchmarks assess the model's holistic understanding, including temporal logic, structured reasoning, and ego-centric perception.

\begin{itemize}
    \item \textbf{VideoMME}~\citep{fu2024videomme}: A comprehensive Multi-modal LLM Evaluation benchmark designed for long-form video understanding. It covers diverse domains including movies, sports, and documentaries, with video lengths ranging from minutes to hours. It assesses the model's ability to process long-context information without subtitles.
    
    \item \textbf{MLVU}~\citep{zhou2025mlvu}: A holistic benchmark specifically tailored for long videos. It includes diverse tasks such as \textit{Topic Reasoning}, \textit{Anomaly Recognition}, \rv{\textit{Ego Reasoning}}, etc. In our experiments, we specifically focus on the \rv{Topic Reasoning (TR) and Ego Reasoning (Ego)} sub-tasks to evaluate the model's sequence understanding and first-person perspective reasoning.
    
    \item \textbf{ETBench}~\citep{liu2024etbench}: A fine-grained benchmark designed to evaluate time-sensitive video understanding. We test the model on 8 different tasks: \textit{Referred Action Recognition (RAR)}, \textit{Event-Caption Alignment (ECA)}, \textit{Referred Video Question-Answering (RVQ)}, \textit{Temporal Video Grounding (TVG)}, \textit{Episodic Memory (EPM)}, \textit{Temporal Action Localization (TAL)}, \textit{Temporal Event Matching (TEM)}, and \textit{Grounded Video Question-Answering (GVQ)}. This benchmark is crucial for verifying whether the model's thinking process translates to accurate event-level deductions.
    
    \item \textbf{NExT-GQA}~\citep{xiao2024nextgqa}: An extension of the NExT-QA dataset that introduces grounded video question answering. It focuses on causal (Why) and temporal (How) questions involving rich object interactions. Unlike standard QA, NExT-GQA requires the model to not only answer the question but also ground the answer in specific video segments, making it an ideal testbed for our reasoning-based approach.
\end{itemize}

\section{Data Construction Pipeline Details}
\label{appendix:data_pipeline}

We build PARSE-60K from $35$K videos sampled from Panda-70M~\citep{chen2024panda70m}, as summarized in Figure~\ref{fig:datapip}. We do not reuse Charades or ActivityNet. Because Panda-70M is also crawled from the web, we further drop every video whose ID appears in any eval set we report on. The pipeline consists of the following stages.

\paragraph{Event segmentation.} Panda-70M has one caption per video and no event timestamps, so we make the timestamps ourselves. Shot detection proposes cut points. We merge or split those cuts so each segment is $10$--$30$ seconds long. \rv{A human then audits these boundaries on a sampled half of the data, using the batch sampling of \cref{appendix:data_pipeline}, and fixes the cases where a camera cut is not an event change, for example one action split by a cut, or two actions in one shot.} These timestamps are what the evidence generator sees.

\paragraph{Pre-processing and filtering.} After the ID filter, we keep videos with many events. We treat one event as 10--30 seconds: for a video of length $T$ we keep it only if $N \ge T/30$, so a 150-second video needs at least 5 events. If two checked segments leave a hole (for example $[25, 38]$ between $[15, 25]$ and $[38, 50]$), we also ask the model to describe that hole. Videos with $N < T/30$ are left out of evidence training, but we still use them in RL.

\paragraph{Evidence generation.} We feed the verified dense timestamps into an evidence generator to produce the event evidence descriptions. Since evidence entries are fine-grained structured signals that require strong instruction following to achieve high data quality, we employ two complementary models: Qwen3-VL-235B-A22B-Instruct~\citep{bai2025qwen3vl} and Gemini-2.5-Pro~\citep{comanici2025gemini-2-5}. To prevent specific model preferences from compromising data quality, the two models alternate roles: when one serves as the generator, the other acts as a quality judge. The two strong models mutually critique and filter outputs, enabling evidence to capture detailed visual cues. The data produced in this stage is served for Stage 1.

\paragraph{Thinking trace generation.} After obtaining the evidence, we generate challenging questions and prompt the LLM to produce explicit thinking traces conditioned on the question and the associated evidence information. Since thinking data demands a higher quality than evidence itself, we use Gemini-2.5-Pro exclusively as the generator in this stage, while Qwen3-VL serves as the quality judge. The data produced in this stage is served for Stage 2.

\paragraph{Structured thinking stages.} Building upon the resulting Structured Event Evidence, we decompose reasoning into stages that progressively narrow, verify, and consolidate evidence. Specifically, we adopt a structured evidence-grounded reasoning process with three tightly coupled stages: 1. Global Search \& Localization, 2. Evidence Verification, and 3. Final Alignment. As shown in~\cref{fig:teaser}, the first phase \textit{Global Search \& Localization} requires the model to retrieve query-relevant clues from the structured evidence list to quickly narrow attention to a small set of highly relevant time windows. In the \textit{Evidence Verification} stage, the model examines temporal coherence by analyzing the preceding and succeeding events around the hypothesized interval, and validates the localization through temporal event consistency rather than shallow pattern matching. To keep decisions in visual evidence, the model performs mandatory visual verification on the inferred interval using vision tokens paired with textual prompts. Before \textit{Final Alignment}, the model performs a global consistency check to ensure the inferred temporal interval is fully contained within the observed activity sequences.

\paragraph{Human evaluation.} To check data quality, we split the generated data into batches of $1{,}000$ annotations and draw $500$ of each batch at random for manual inspection. \rv{This covers $30$K of the $60$K annotations, and the boundary audit above draws on the same sampled half, so all human review in the pipeline follows one protocol. The two thresholds then do different jobs: the automatic judge in \cref{fig:prompt_judge} gates a single sample, while the human score below gates a whole batch.} Three annotators score each sampled instance from 1 to 10, following the same criteria as the automatic judge but they are more accurate and strict (\cref{fig:prompt_judge}). Their average scores are $8.5$, $7.9$ and $8.2$, with $87.1\%$ of samples within one point.

\textit{Format completeness.} For evidence, the \texttt{<evidence>} tag must be closed, the global theme line must be present, and every time segment must fill in all six fields (time, person, human action, scene, object, camera), with one line per input segment and no missing or invented timestamps. For thinking traces, \texttt{<evidence>}, \texttt{<thinking>} and \texttt{<answering>} must all be present, and the thinking block must actually walk through Global Search, Evidence Verification and Final Alignment.

\textit{Fluency.} Whether the text reads naturally and says something specific. Evidence has to be concrete enough to check against the video, not ``a person is doing something''. In a thinking trace, each step has to move the reasoning forward rather than restate the conclusion in new words.

\textit{Hallucination.} Whether everything stated can be found in the video. For evidence we look for entities, actions or timestamps that are not there. For thinking traces the bar is higher: on top of not inventing visual content, a trace may not cite a segment that is absent from the evidence list, and may not jump to the answer when the evidence does not support it.

We anchor the scale as follows. $9$--$10$: all three are clean and the sample is kept as is. $7$--$8$: minor wording problems, such as a vague or slightly repetitive description, but no hallucination and complete format. $4$--$6$: missing format fields or mild hallucination, for instance one attribute that does not match the video. $1$--$3$: the format is badly broken or large parts are invented. Any sample scoring below $6$ sends its whole batch back for regeneration, which happened to $28.3\%$ of batches.

The two data types are weighted differently. Evidence is scored mainly on format completeness and timestamp accuracy, since it is the input to downstream reasoning and one wrong timestamp propagates along the chain. Thinking traces get their format largely for free from the generation template, so they are scored mainly on hallucination and on whether the reasoning is valid; this is also why we use only Gemini-2.5-Pro to generate them.

\section{More details of PARSE-60K}
\label{appendix:data_details}
\subsection{Video Source}
We sample $35$K videos from Panda-70M~\citep{chen2024panda70m} and keep only the raw video, not the original captions. We do not train on Charades-STA~\citep{gao2017charades} or ActivityNet~\citep{krishna2017dense}, because we test on those sets. Panda-70M is also from the web, so a different source is not enough: we match video IDs against every eval set and drop any hit. After this filter, PARSE-60K and the eval sets share no video.

The downside is that we have to make event timestamps ourselves. As in Appendix~\ref{appendix:data_pipeline}, we use shot detection, merge or split to $10$--$30$ seconds, and \rv{let a human audit half of the boundaries. The timestamps follow our event length, not the style of any benchmark.}

\subsection{RL Training Data Details}
\label{sec:rl_data_details}

In this section, we provide a detailed statistical analysis of the dataset used for the Reinforcement Learning (RL) stage. The RL stage reuses the full corpus of $35$K videos, specifically curated to optimize temporal reasoning and grounding capabilities through diverse task types and semantic domains.

\textbf{Task Type Composition.} As illustrated in Figure \ref{fig:rl_tasks}, the dataset follows a specific distribution across different task types to balance reward stability and semantic understanding. The majority of the dataset is dedicated to \textit{Temporal Grounding} ($53\%$). This high proportion is designed to leverage the objective nature of Temporal Intersection over Union (TIoU) metrics, which provide dense and stable reward signals during the policy optimization process. To ensure the model maintains robust logical reasoning capabilities beyond pure localization, the dataset includes a substantial portion of Visual Question Answering tasks, specifically \textit{Spatial VQA} ($21\%$) and \textit{Reasoning VQA} ($20\%$). These tasks require the model to interpret visual relationships and perform structured deductions. The remaining portion consists of \textit{Temporal VQA} ($6\%$). Every RL prompt therefore has a task metric, either temporal IoU for grounding or multiple-choice accuracy for the VQA types, so no group is left with only the format and length objectives.

Since the RL stage reuses the full corpus, only the task type assigned to each video differs from the evidence and thinking stages.

\begin{figure}[H]
  \centering
  \includegraphics[width=0.5\linewidth]{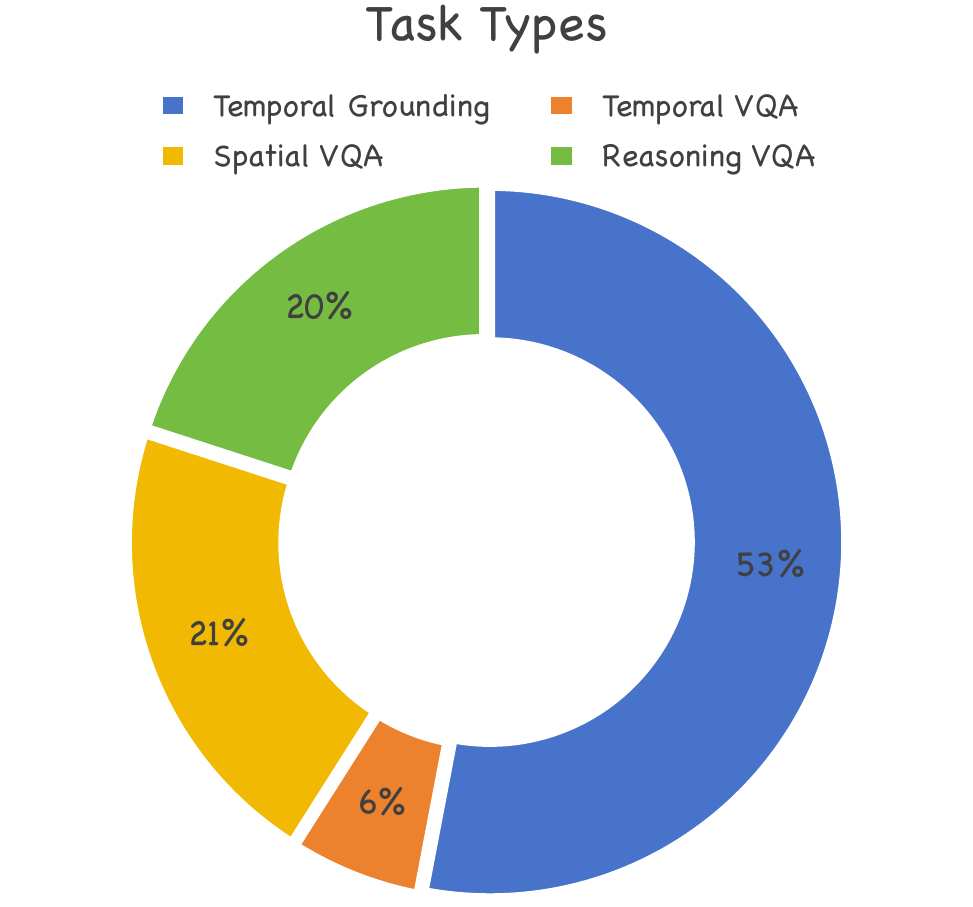}
  \caption{\rv{Distribution of task types in the RL training data. \textit{Temporal Grounding} constitutes the largest share ($53\%$), providing precise IoU-based supervision, while \textit{Spatial VQA} ($21\%$), \textit{Reasoning VQA} ($20\%$) and \textit{Temporal VQA} ($6\%$) broaden the training signal toward spatial, semantic and temporal understanding.}}
  \label{fig:rl_tasks}
\end{figure}

\subsection{Data Curation Prompts}
\label{appendix:data_curation_prompts}

This section describes the prompts used in our data curation pipeline. We employ three specialized prompts for different stages: Evidence Generation, Quality Judge, and Thinking Generation. The complete prompts are illustrated in Figure~\ref{fig:prompt_fact}, Figure~\ref{fig:prompt_judge}, and Figure~\ref{fig:prompt_think}, respectively.

\paragraph{Evidence Generation Prompt.} The Evidence Generation prompt (Figure~\ref{fig:prompt_fact}) instructs the model to analyze video content and produce structured summaries for pre-defined time segments. Given a list of timestamps as input, the model generates a theme summary and event evidence for each segment. The prompt enforces a strict schema with six required components per segment: \texttt{[person]}, \texttt{[human\_action]}, \texttt{[scene]}, \texttt{[object]}, \texttt{[camera]}, and \texttt{[event\_caption]}. A critical dependency rule ensures logical consistency: if no person is present in the scene, the human action field must be marked as ``None''. The prompt also specifies that timestamps must be preserved in their original seconds format without conversion to MM:SS, and all output must be enclosed within \texttt{<evidence></evidence>} tags.

\paragraph{Quality Judge Prompt.} The Quality Judge prompt (Figure~\ref{fig:prompt_judge}) defines a strict evaluation protocol for validating generated evidence. It takes two inputs: the target timestamps that must be present and the model-generated structured text. The evaluation criteria are organized into four categories: (1) \textit{Format Adherence} ensures proper tag structure, required headers, and all six mandatory fields in every line; (2) \textit{Timestamp Alignment} verifies one-to-one correspondence between output lines and input segments, rejecting any missing or hallucinated timestamps; (3) \textit{Logical Consistency} enforces the person-action dependency rule, ensuring that ``None'' person entries always correspond to ``None'' actions; and (4) \textit{Content Quality} assesses completeness, fluency of captions, and specificity of descriptions. The judge outputs a structured JSON containing a numeric score (0-10), a regeneration flag, a list of critical errors, and qualitative feedback.

\paragraph{Thinking Generation Prompt.} The Thinking Generation prompt (Figure~\ref{fig:prompt_think}) guides the model to produce evidence-based reasoning chains for video question answering. Given a user question and extracted evidence, the model must derive answers through the \rv{three-stage} structured reasoning process \rv{of \cref{appendix:data_pipeline}}. In \textit{Stage 1: Global Search \& Localization}, the model scans evidence descriptions for keywords related to the question, identifies candidate time segments based solely on textual evidence, and explicitly states which evidence IDs are selected as candidates. In \textit{Stage 2: Evidence Verification}, the model validates candidates using antecedent-action-consequence logic: \texttt{[Antecedent]} establishes the pre-condition contrast, \texttt{[Verifying Key Visual Actions]} provides textual proof of the event, and \texttt{[Consequence]} confirms the post-condition result. \rv{In \textit{Stage 3: Final Alignment}, it checks that the selected interval is contained within the observed activity sequence before committing to an answer.} The output follows a \texttt{<thinking>...</thinking><answer>...</answer>} format, ensuring that reasoning is grounded in explicit structured evidence rather than implicit assumptions.

\section{Instruction Tuning Data Format}
\label{sec:prompts}

To effectively train the model to reason over structured evidence, we design a progressive prompt engineering strategy that guides the generation of training data across two distinct stages. This multi-stage approach ensures the model first learns to extract high-density visual evidence, and then integrates this skill into a coherent reasoning chain.

In the initial stage, the objective is to align the model's visual perception with our structured event schema, as displayed in~\cref{fig:prompt_stage1}. We explicitly instruct the model to localize activity events and generate a detailed report containing a global theme and detailed event evidence. We enforce a rigid output format for each temporal segment, requiring specific fields such as time, person, human action, scene, object, camera, and an event caption. This design forces the model to decompose complex video dynamics into discrete, verifiable atomic elements rather than generating generic captions, effectively transforming raw video data into the structured representation used throughout the pipeline.

The final stage focuses on generating high-quality CoT data, as depicted in~\cref{fig:prompt_stage2}. We combine the structured evidence from the first stage with challenging QA pairs. The prompt requires the model to first recall the structured evidence context. Crucially, the subsequent thinking block is no longer trivial; it contains a detailed derivation process where the model performs global search, evidence verification, and alignment based on the provided evidence. This stage produces the \rv{reference} thinking data used to warm-start the policy model, ensuring that the reasoning process is explicitly grounded in the extracted evidence.

\section{More details of Reinforcement Learning}
\subsection{Detailed Rewards Design}
\label{sec:reward_design}

To guide the model toward generating structured, evidence-grounded, and concise reasoning chains, we design a composite reward function consisting of four distinct components: \textit{Format Reward, Linear IoU Reward, Multi-choice Accuracy Reward, and Length Reward}.

\paragraph{Structured Thinking Format Reward ($r_{\text{fmt}}$).}
To enforce the structured reasoning schema proposed in our method, we design a rule-based format reward. This reward evaluates whether the model strictly follows the \texttt{<evidence>}, \texttt{<thinking>}, and \texttt{<answering>} XML structure and whether the \texttt{<thinking>} block explicitly traverses the required structured reasoning steps.

Let $\mathcal{K}$ be the set of mandatory structural keywords defined as:
\begin{equation}
\label{eq:keyword_set}
\begin{aligned}
    \mathcal{K} = \{ & \textit{``Global Search'', ``Evidence Verification'', ``Final Alignment'',} \\
                     & \textit{``Antecedent'', ``Visual Verification'', ``Consequence''} \}
\end{aligned}
\end{equation}

Let $\mathcal{P}_{\text{basic}}$ denote the satisfaction of the basic regex pattern, $\mathcal{T}_{\text{valid}}$ denote the constraint that all XML tags appear exactly once, and $\mathcal{C}_{\text{think}}$ represent the content extracted from the \texttt{<thinking>} block. The format reward $r_{\text{fmt}}$ is strictly defined as:

\begin{equation}
\label{eq:format_reward}
r_{\text{fmt}} = 
\begin{cases} 
1.0 & \text{if } \mathcal{P}_{\text{basic}} \land \mathcal{T}_{\text{valid}} \land (\mathcal{K} \subseteq \mathcal{C}_{\text{think}}) \\
0.5 & \text{if } \mathcal{P}_{\text{basic}} \land [\neg \mathcal{T}_{\text{valid}} \lor (\mathcal{K} \not\subseteq \mathcal{C}_{\text{think}})] \\
0.0 & \text{otherwise (pattern mismatch)}
\end{cases}
\end{equation}

This reward scheme assigns a baseline score of 0.5 for adhering to the basic format structure, grants a full score of 1.0 only when the reasoning chain is structurally complete and structurally rigorous, and assigns 0.0 for format violations.

\paragraph{Task-Specific Performance Rewards ($r_{\text{task}}$).}
Depending on the downstream task type, we apply different performance metrics as rewards:

\textit{Linear IoU Reward:} For temporal grounding tasks, we compute the Temporal Intersection over Union (TIoU) between the predicted time segments $S_{pred}$ and ground truth segments $S_{gt}$. To handle ambiguity where a single continuous prediction covers multiple disjoint ground truth segments, we implement a hybrid scoring mechanism. Specifically, if $|S_{pred}|=1$ and $|S_{gt}| > 1$, the reward is defined as $\max(\text{Coverage Ratio}, \text{Span IoU})$, ensuring the model is rewarded for correctly identifying the overall event span. Otherwise, we compute the average IoU across all segment pairs.

\textit{Multi-choice Accuracy Reward:} For QA tasks, we strictly parse the predicted option (A/B/C/D) from the \texttt{<answering>} block. The reward is a binary indicator function: $r_{\text{acc}} = \mathbb{I}(\text{Pred} = \text{GT})$.

Every rollout is therefore scored on three objectives at once, namely schema compliance, the task metric, and the reasoning budget. A prompt is either a grounding prompt or a QA prompt, so only one of the two task metrics is active on any given prompt, and \method{} balances format, task metric and budget inside each group.

\paragraph{Length Reward ($r_{\text{len}}$).}
To prevent the ``verbosity tax'' often associated with CoT reasoning, we introduce a soft length constraint. Unlike hard truncation, we employ a piecewise linear decay function to gently discourage excessive token generation beyond a target budget. We penalize only over-long outputs, for two reasons. Hard truncation gives zero reward to every over-long rollout, so early in training many rollouts collapse to the same score and the group-relative advantage carries little information. A two-sided penalty would pay the model for being short, which is exactly the shortcut that leads to cutting too much~\citep{wang2026icr, li2026leash}. Our reward is flat below $L_{target}$, so being shorter than the budget earns nothing extra. This shape only limits how the conflict can appear; how much $r_{\text{len}}$ counts against $r_{\text{acc}}$ is still the real decision, and that is what \method{} sets automatically.

\rv{The budget binds during training. Stage-3 rollouts are drawn from a corpus that mixes grounding and QA prompts and average roughly $2{,}500$ to $3{,}100$ tokens, inside the $(L_{target}, L_{max}]$ band where $r_{\text{len}}$ decays, so the length column varies within a group and survives the variance filter.}

Let $L$ be the generated response length, $L_{max}$ be the hard maximum length, and $L_{buffer}$ be a soft buffer zone. The target length is defined as $L_{target} = L_{max} - L_{buffer}$. \rv{We set $L_{max} = 4096$ and $L_{buffer} = 2500$, so the penalty starts at $L_{target} = 1596$ tokens.} The length reward is calculated as:

\begin{equation}
\label{eq:length_reward}
r_{\text{len}}(L) = 
\begin{cases} 
1.0 & \text{if } L \le L_{target} \\
1.0 - \frac{L - L_{target}}{L_{buffer}} & \text{if } L_{target} < L \le L_{max} \\
0.0 & \text{if } L > L_{max}
\end{cases}
\end{equation}

This design provides a full reward for concise answers within the target budget and applies a linear penalty only when the model enters the buffer zone, preventing the optimization instability caused by abrupt reward cliffs.

\subsection{The \method{} Algorithm}
\label{appendix:grab_algorithm}

Algorithm~\ref{alg:pfab_groupwise} states the full procedure that \cref{sec:pfab} describes in text.

\begin{algorithm}[ht]
\caption{\method{}}
\label{alg:pfab_groupwise}
\begin{algorithmic}[1]
\STATE \textbf{Input:}
\STATE \hspace{1.5em} $\mathbf{R} \in \mathbb{R}^{N_{\text{total}} \times M}$: rewards for $N_{\text{total}}$ responses and $M$ objectives
\STATE \hspace{1.5em} $\mathbf{g} \in \mathbb{Z}^{N_{\text{total}}}$: group (prompt) ID for each response
\STATE \hspace{1.5em} $\tau_{\sigma} > 0$: standard deviation threshold (e.g., $10^{-6}$)
\STATE \hspace{1.5em} $\eta > 0$: min-norm floor (e.g., $10^{-3}$)
\STATE \textbf{Output:}
\STATE \hspace{1.5em} $\mathbf{A} \in \mathbb{R}^{N_{\text{total}}}$: scalarized advantage for each response

\STATE Initialize $\mathbf{A} \gets \mathbf{0}_{N_{\text{total}}}$
\FOR{each unique group $q \in \mathrm{unique}(\mathbf{g})$}
    \STATE \textit{// Step 1: Rewards Centralization (Eq.~\ref{eq:relative_advantage})}
    \STATE $\mathcal{I}_q \gets \{i : g_i = q\}$ \hfill $\triangleright$ Indices in group $q$
    \STATE $\mathbf{D}_q \gets \mathbf{R}_{\mathcal{I}_q,:} - \mathrm{mean}(\mathbf{R}_{\mathcal{I}_q,:}, \text{axis}=0)$
    
    \STATE \textit{// Step 2: Standardized Conflict Resolution (Eq.~\ref{eq:mgda_opt_group})}
    \STATE $\boldsymbol{\sigma}_q \gets \mathrm{std}(\mathbf{D}_q, \text{axis}=0)$
    \STATE $\mathcal{V} \gets \{m : \sigma_{q,m} > \tau_{\sigma}\}$ \hfill $\triangleright$ Live objectives
    \IF{$\mathcal{V} = \emptyset$}
        \STATE $\boldsymbol{\alpha}^* \gets \frac{1}{M}\mathbf{1}_M$ \hfill $\triangleright$ Fallback to uniform
    \ELSE
        \STATE $\hat{\mathbf{D}}_q \gets \mathbf{D}_q[:, \mathcal{V}] / \boldsymbol{\sigma}_q[\mathcal{V}]$ \hfill $\triangleright$ Standardization
        \STATE Solve:
        $\boldsymbol{\alpha}_{\mathcal{V}}^* = \arg\min_{\boldsymbol{\alpha}} \|\hat{\mathbf{D}}_q\boldsymbol{\alpha}\|^2$ s.t. $\boldsymbol{\alpha} \in \Delta_{|\mathcal{V}|}$ 
        via Frank-Wolfe
        \IF{$\|\hat{\mathbf{D}}_q\boldsymbol{\alpha}_{\mathcal{V}}^*\| < \eta$}
            \STATE $\boldsymbol{\alpha}_{\mathcal{V}}^* \gets \frac{1}{|\mathcal{V}|}\mathbf{1}$ \hfill $\triangleright$ near-zero min-norm $\to$ uniform
        \ENDIF
        \STATE $\boldsymbol{\alpha}^* \gets \mathbf{0}_M$; \quad $\boldsymbol{\alpha}^*[\mathcal{V}] \gets \boldsymbol{\alpha}_{\mathcal{V}}^*$
    \ENDIF

    \STATE \textit{// Step 3: Advantage Scalarization (Eq.~\ref{eq:final_advantage_group})}
    \STATE $\mathbf{A}^{\text{raw}}_q \gets \mathbf{D}_q \boldsymbol{\alpha}^*$ \hfill $\triangleright$ Apply weights to centered rewards
    \STATE $\mathbf{A}[\mathcal{I}_q] \gets (\mathbf{A}^{\text{raw}}_q - \mathrm{mean}(\mathbf{A}^{\text{raw}}_q)) / (\mathrm{std}(\mathbf{A}^{\text{raw}}_q) + \epsilon)$
\ENDFOR
\STATE \textbf{return} $\mathbf{A}$
\end{algorithmic}
\end{algorithm}

\subsection{Live objectives per group}
\label{sec:live_objectives}

\rv{A prompt is either a grounding prompt or a QA prompt, so one of the two task metrics is constant within any group and is removed by the variance filter, leaving at most three live objectives. When exactly two are live, the two standardized columns have equal norm, so the shortest point on the segment between them is its midpoint and $\boldsymbol{\alpha}^{*}_q = (\tfrac{1}{2}, \tfrac{1}{2})$; the factor cancels in the $z$-score of Eq.~\ref{eq:final_advantage_group}. The adaptive weights therefore matter most for groups in which all three objectives vary.}

\subsection{Frank-Wolfe Algorithm for Solving the Minimum-Norm Problem}
\label{appendix:frank_wolfe}

To solve the minimum-norm optimization problem in Eq.~\eqref{eq:mgda_opt_group}, we employ the Frank-Wolfe algorithm, also known as the conditional gradient method~\citep{jaggi2013revisitingfrankwolfe}, which is particularly well-suited for constrained convex optimization over polytopes such as the probability simplex.

\paragraph{Problem Formulation.}
Recall that our objective is to find the optimal weight vector $\boldsymbol{\alpha}^{*}$ that minimizes the squared norm of the aggregated direction in the standardized space:
\begin{equation}
\boldsymbol{\alpha}^{*} = \arg\min_{\boldsymbol{\alpha} \in \Delta_M} f(\boldsymbol{\alpha}), \quad \text{where} \quad f(\boldsymbol{\alpha}) = \left\| \hat{\mathbf{D}}_q \boldsymbol{\alpha} \right\|^{2}.
\end{equation}
This is a convex quadratic program over the simplex $\Delta_M = \{\boldsymbol{\alpha} \in \mathbb{R}^M \mid \sum_{m=1}^{M} \alpha_m = 1, \alpha_m \ge 0\}$.

\paragraph{Algorithm Description.}
Starting from a uniform initialization $\boldsymbol{\alpha}^{(0)} = \frac{1}{M}\mathbf{1}_M$, each iteration $t$ proceeds as follows:

\textit{Step 1: Gradient Computation.} We compute the gradient of the objective at the current iterate:
\begin{equation}
\nabla f(\boldsymbol{\alpha}^{(t)}) = 2\hat{\mathbf{D}}_q^\top (\hat{\mathbf{D}}_q \boldsymbol{\alpha}^{(t)}).
\end{equation}

\textit{Step 2: Linear Minimization Oracle (LMO).} We solve a linear minimization problem over the simplex:
\begin{equation}
\mathbf{s}^{(t)} = \arg\min_{\mathbf{s} \in \Delta_M} \langle \mathbf{s}, \nabla f(\boldsymbol{\alpha}^{(t)}) \rangle.
\end{equation}
Since $\Delta_M$ is a simplex, the solution is simply a one-hot vector placing all mass on the coordinate with the smallest gradient component:
\begin{equation}
\mathbf{s}^{(t)} = \mathbf{e}_{k^*}, \quad \text{where} \quad k^* = \arg\min_{m} [\nabla f(\boldsymbol{\alpha}^{(t)})]_m.
\end{equation}

\textit{Step 3: Exact Line Search.} We compute the update direction $\mathbf{d}^{(t)} = \mathbf{s}^{(t)} - \boldsymbol{\alpha}^{(t)}$ and determine the optimal step size via exact line search. For the quadratic objective, this admits a closed-form solution:
\begin{equation}
\gamma^{(t)} = \max\left(0, \min\left(1, \frac{- \langle \hat{\mathbf{D}}_q \boldsymbol{\alpha}^{(t)}, \hat{\mathbf{D}}_q \mathbf{d}^{(t)} \rangle}{\| \hat{\mathbf{D}}_q \mathbf{d}^{(t)} \|^2} \right)\right).
\end{equation}
If $\|\hat{\mathbf{D}}_q \mathbf{d}^{(t)}\|^2 < \epsilon_{\text{fw}}$ (we use $\epsilon_{\text{fw}} = 10^{-12}$), we terminate early to avoid numerical instability.

\textit{Step 4: Weight Update.} We update the weights via:
\begin{equation}
\boldsymbol{\alpha}^{(t+1)} = \boldsymbol{\alpha}^{(t)} + \gamma^{(t)} \mathbf{d}^{(t)}.
\end{equation}

\paragraph{Termination Criteria.}
The algorithm terminates when the change in weights falls below a tolerance threshold \rv{$\|\boldsymbol{\alpha}^{(t+1)} - \boldsymbol{\alpha}^{(t)}\| < \tau_{\mathrm{fw}}$ (we use $\tau_{\mathrm{fw}} = 10^{-6}$)}, or when the maximum number of iterations is reached (we use $T_{\max} = 50$). \rv{This tolerance is separate from the variance threshold $\tau_{\sigma}$ in Algorithm~\ref{alg:pfab_groupwise}.}

\paragraph{Near-zero min-norm fallback.}
\rv{Frank--Wolfe can return a mix whose direction $\|\hat{\mathbf{D}}_q\boldsymbol{\alpha}^*\|$ is close to zero. Every live column has the same norm, so this arises only when the columns cancel, never when they agree, and Algorithm~\ref{alg:pfab_groupwise} then falls back to uniform weights on $\mathcal{V}$ ($\eta = 10^{-3}$).}
%

\paragraph{\rv{Convergence of the Frank--Wolfe Solver.}}
The exact line search ensures monotonic decrease of the objective value at every iteration, guaranteeing convergence to the global optimum of this convex problem, \rv{which by Eq.~\eqref{eq:alignment} is the weight vector minimizing $\boldsymbol{\alpha}^\top \mathbf{C}_q \boldsymbol{\alpha}$ over the live objectives of that group.}

\section{Additional Experiments and Analyses}
\label{sec:additional_analysis}

\subsection{TempCompass and EventBench Evaluation}
\label{sec:tempcompass_eventbench}

To further validate the temporal and event-level reasoning capabilities of our model, we evaluate on two additional benchmarks: TempCompass, which assesses temporal perception including event ordering and attribute change detection, and EventBench, which evaluates event-level understanding including temporal and causal reasoning sub-tasks. As shown in Table~\ref{tab:tempcompass_eventbench}, PARSE-4B achieves 70.64\% on TempCompass and 63.48\% on EventBench, outperforming both Qwen3-VL-4B-Instruct and Qwen3-VL-4B-Thinking across all sub-tasks. Notably, the largest improvements appear in attribute change detection (+3.45\%) and causal reasoning (+2.60\%), confirming that structured event evidence effectively enhances the model's ability to capture fine-grained temporal dynamics and event dependencies.

\input{tables/tempcompass_eventbench}

\subsection{Scalability to 8B Model}
\label{sec:8b_experiment}

To validate that \method{} generalizes beyond the 4B scale, we train Qwen3-VL-8B with LoRA using both standard GRPO and our \method{} algorithm. As shown in Table~\ref{tab:8b}, \method{} outperforms GRPO on all four benchmarks, with the largest gains on VideoMME and \rv{CTL-mIoU, the mean IoU on Charades-TimeLens}. \rv{Our compute budget at this scale was limited, so both 8B runs use half of the RL data we use at 4B and a lower LoRA rank. Both runs carry the same setting, so the comparison between GRPO and \method{} is made under matched conditions.}

\begin{table}[H]
\centering
\resizebox{0.75\linewidth}{!}{%
\begin{tabular}{@{}lcccc@{}}
\toprule
\textbf{Model (8B)} & \textbf{VideoMME} & \textbf{CTL-mIoU} & \textbf{TempCompass} & \textbf{EventBench} \\ \midrule
8B-GRPO & 64.81 & 43.73 & 65.79 & 64.87 \\
\textbf{8B-\method{}} & \textbf{66.98} & \textbf{46.81} & \textbf{66.23} & \textbf{66.10} \\ \bottomrule
\end{tabular}
}
\vspace{1em}
\caption{\method{} vs. GRPO on the 8B model scale (LoRA fine-tuning). \method{} outperforms GRPO on all four benchmarks, \rv{so the margin is not confined to 4B. Both runs use half the RL data and a lower LoRA rank than the 4B experiments, so the absolute numbers sit below the 4B model and carry no scaling claim.} \rv{CTL-mIoU is the mean IoU on Charades-TimeLens.}}
\label{tab:8b}
\end{table}

\subsection{\method{} Computational Overhead}
\label{sec:pfab_overhead}

A key concern for multi-objective optimization algorithms is their computational cost. We measure the wall-clock time of the \method{} weight computation step and compare it against the total training step time. As shown in Table~\ref{tab:pfab_overhead}, \method{} adds negligible overhead ($0.0013\%$ of total step time) across all group sizes, as the core computation involves only a small quadratic program over the number of objectives ($M=4$), independent of the model size.

\begin{table}[H]
\centering
\resizebox{0.5\linewidth}{!}{%
\begin{tabular}{@{}ccc@{}}
\toprule
\textbf{Group Size} & \textbf{\method{} (ms)} & \textbf{One Step (ms)} \\ \midrule
4  & 0.68 & 52,789 \\
8  & 0.69 & 108,401 \\
16 & 1.00 & 211,389 \\ \bottomrule
\end{tabular}
}
\vspace{1em}
\caption{\method{} computational overhead. The weight computation is negligible ($<0.002\%$) relative to the total training step time.}
\label{tab:pfab_overhead}
\end{table}

\subsection{Evidence Quality Evaluation}
\label{sec:fact_quality}

\rv{These metrics measure agreement with the annotations used for supervision.} We compute standard text generation metrics (F1, BLEU-4, ROUGE-L, CIDEr) \rv{against the held-out evidence annotations of PARSE-60K, where F1 is the token-level $F_1$ between the generated and reference evidence blocks after field-wise alignment}. As shown in Table~\ref{tab:fact_quality}, PARSE-4B substantially outperforms the base Instruct model across all metrics, confirming that our multi-stage training pipeline effectively teaches the model to produce structured evidence in the target format.

\begin{table}[H]
\centering
\resizebox{0.65\linewidth}{!}{%
\begin{tabular}{@{}lcccc@{}}
\toprule
\textbf{Model} & \textbf{F1} & \textbf{BLEU-4} & \textbf{ROUGE-L} & \textbf{CIDEr} \\ \midrule
Qwen3-VL-4B-Instruct & 44.2 & 10.6 & 27.1 & 4.1 \\
\textbf{PARSE-4B} & \textbf{54.7} & \textbf{16.6} & \textbf{35.1} & \textbf{17.3} \\ \bottomrule
\end{tabular}
}
\vspace{1em}
\caption{Evidence quality evaluation. PARSE-4B generates substantially higher-quality structured evidence compared to the base model.}
\label{tab:fact_quality}
\end{table}

\subsection{Lagrangian (Length) Baseline}
\label{sec:lagrangian_baseline}

This baseline replaces the static length coefficient with the online dual variable of~\citet{li2026leash}, so it adapts the accuracy--length trade-off during training with one global value. Everything else matches the \method{} run: same Stage-2 checkpoint, same RL corpus, same four rewards, $\mathrm{num\_generation}=16$, and the same learning rate and step count. The advantage is plain GRPO, that is, weighted sum followed by the group $z$-score.

The scalar reward is $R = w_{\text{fmt}} r_{\text{fmt}} + w_{\text{task}} r_{\text{task}} + \lambda\, r_{\text{len}}$, where $w_{\text{fmt}}$ and $w_{\text{task}}$ are frozen at the Fixed Weight (Accuracy) values and only $\lambda$ moves. After each step we gather the mean completion length $\bar L$ over the whole batch and update
\begin{equation}
\lambda \leftarrow \mathrm{clip}\left(\lambda + \eta_\lambda \cdot \frac{\bar L - L_{\text{target}}}{L_{\text{target}}},\; 0,\; \lambda_{\max}\right),
\end{equation}
with $\lambda_0 = 1$, dual step size $\eta_\lambda = 0.01$ and $\lambda_{\max} = 5$. We drive the update from $\bar L$ rather than from $\bar r_{\text{len}}$ because $r_{\text{len}}$ is flat at $1.0$ below $L_{\text{target}}$ and would carry no signal once the policy is inside the budget. We tune $L_{\text{target}}$ so that the run ends \rv{within a few percent of \method{}'s output length, landing at $2{,}500$ training tokens,} which leaves accuracy as the only axis on which the two differ.

\subsection{\rv{Static Weight Sweep}}
\label{sec:weight_sweep}

\rv{Fixed Weight (Accuracy) in Table~\ref{tab:ablation} is one setting of the static weights, and one point cannot tell us whether \method{} beats static weighting in general or only beats that point. The static family is smaller than it looks, because the weighted sum is $z$-scored inside each group. Multiplying the whole weight vector by a constant therefore leaves the advantage unchanged, so with three objectives able to be live the family has two free ratios rather than three weights. We sweep the one the length reward controls, holding $w_{\text{fmt}}$ and $w_{\text{task}}$ at their tuned values of $1$ and $2$ and varying $w_{\text{len}}$ alone, which moves $w_{\text{len}}/w_{\text{task}}$ from $0$ to $2$.}

\rv{Every run in the sweep starts from the same Stage-2 checkpoint and shares the RL corpus, the learning rate, the step count and $\mathrm{num\_generation}=16$, so the runs differ in $w_{\text{len}}$ and nothing else. We evaluate on Charades-TimeLens alone, whose videos average $29.3$ seconds, since sweeping on the longer benchmarks is outside our compute budget. Two rows come for free, since Fixed Weight (Accuracy) and GRPO are both members of the static family and both already appear in Table~\ref{tab:ablation} alongside \method{} at this group size. The two new runs take the Fixed Weight (Accuracy) vector and change $w_{\text{len}}$ alone: setting it to $0$ shows what accuracy a free budget buys, and $4\times$ its tuned value gives the short end, so the pair brackets \method{}'s output length.}

\rv{\cref{tab:weight_sweep} lists the static runs ordered by output length, and what it shows first is how much the choice of weights matters: accuracy spans $37.8$ to $42.9$ R1@0.5 and length spans $468$ to $1{,}180$ tokens over four settings of a single ratio. The most accurate of them spends $931$ tokens against \method{}'s $579$, and reaching it took four runs.}

\input{tables/weight_sweep}

\subsection{\rv{Seed Variance}}
\label{sec:seed_variance}

\rv{Several margins in Table~\ref{tab:ablation} are a point or two, so we report the spread over the three seeds on Charades-TimeLens for GRPO, Fixed Weight (Accuracy) and \method{} at $G = 16$. Every number we report elsewhere is the mean of these same three seeds.}

\input{tables/seed_variance}

\subsection{\rv{The Weights the Solver Returns}}
\label{sec:measured_weights}

\rv{\cref{tab:measured_weights} reports the weights the solver returns, averaged over the run and conditioned on the size of the live set.}

\rv{In the three-objective groups the weights average $0.15 \pm 0.11$ on format, $0.39 \pm 0.06$ on the task metric and $0.46 \pm 0.07$ on length, with the spread taken across groups, where uniform weighting would put $1/3$ on each. Length carrying the largest average weight is what the criterion of Eq.~\ref{eq:alignment} does when length is the column that moves against the others, and the operating point it produces is the one in \cref{tab:weight_sweep}: $43.4$ R1@0.5 at $579$ tokens. In the two-objective groups the weight is $0.50$ with no spread, the value the algebra of \cref{sec:live_objectives} gives. The last two rows compare \method{} with GRPO checkpoint by checkpoint at the same training step, over the checkpoints through step $100$.}
%

\input{tables/measured_weights}

\subsection{Reasoning Quality Evaluation}
\label{sec:reasoning_quality}

To evaluate the quality of the intermediate reasoning process (beyond final-answer accuracy), we employ Gemini-2.5-Pro as an automatic judge to rate model outputs on four dimensions: Overall quality, Hallucination avoidance, Logicality of the reasoning chain, and Interpretability of the structured output. Each dimension is scored on a 1--5 scale. As shown in Table~\ref{tab:reasoning_quality}, PARSE-4B matches or outperforms the Thinking baseline on all dimensions, with particularly strong gains in Interpretability ($+1.0$) and Logicality ($+0.25$), validating that structured event evidence enhances the reasoning process quality.

\begin{table}[H]
\centering
\resizebox{0.7\linewidth}{!}{%
\begin{tabular}{@{}lcccc@{}}
\toprule
\textbf{Model} & \textbf{Overall} & \textbf{Halluc.} & \textbf{Logic.} & \textbf{Interp.} \\ \midrule
Qwen3-VL-4B-Thinking & 4.00 & 4.0 & 4.25 & 3.75 \\
\textbf{PARSE-4B} & \textbf{4.41} & \textbf{4.0} & \textbf{4.50} & \textbf{4.75} \\ \bottomrule
\end{tabular}
}
\vspace{1em}
\caption{Reasoning quality evaluation by Gemini-2.5-Pro. PARSE-4B produces more logical and interpretable reasoning chains.}
\label{tab:reasoning_quality}
\end{table}

\clearpage
\begin{figure}[H]
\centering
    \begin{minipage}[b]{1.0\linewidth}
        \includegraphics[width=\linewidth]{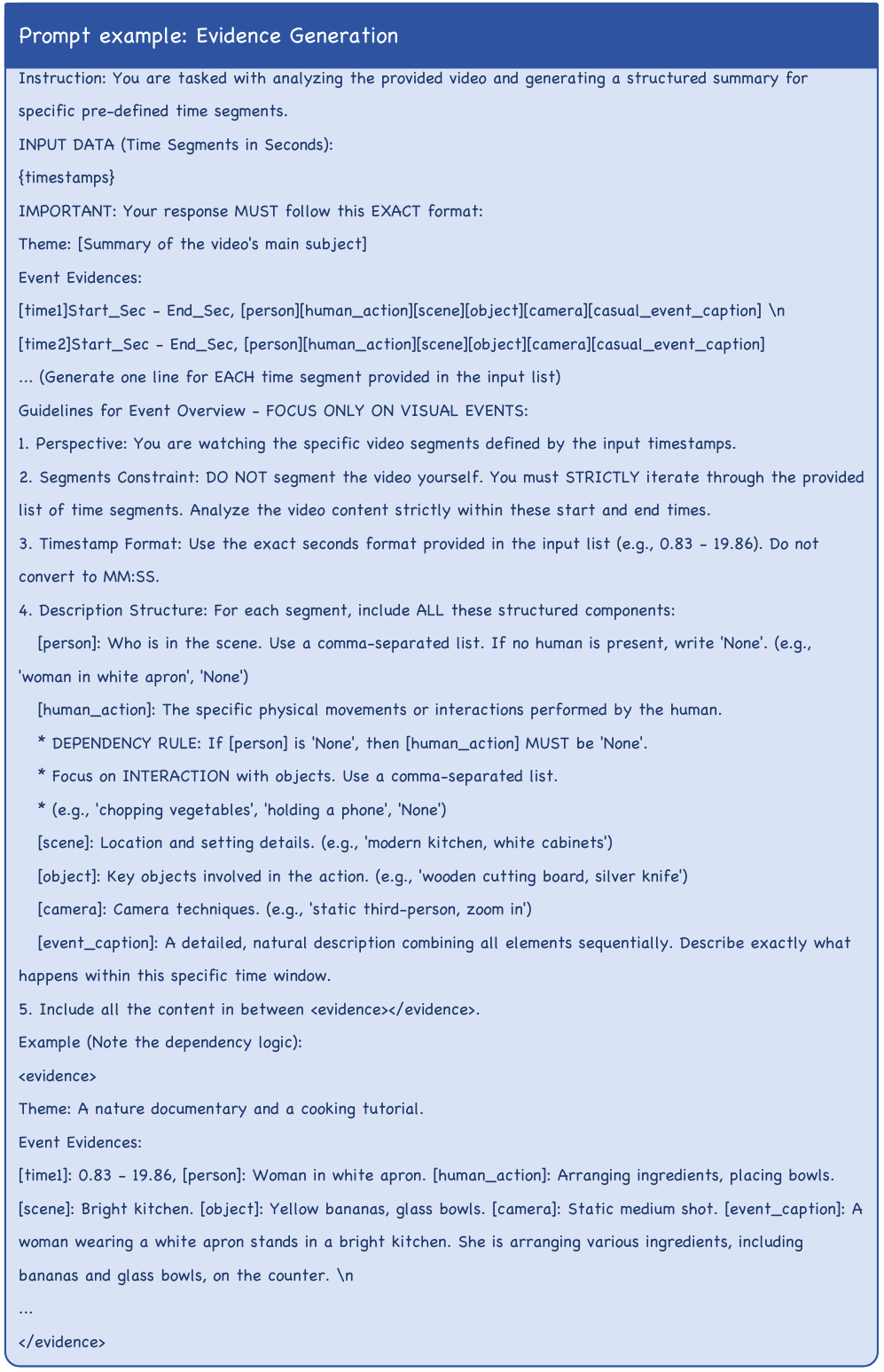}
    \end{minipage}
    \caption{Evidence Generation prompt used in the data curation pipeline to produce structured event evidence from videos.}
    \label{fig:prompt_fact}
\end{figure}

\begin{figure}[H]
\centering
    \begin{minipage}[b]{1.0\linewidth}
        \includegraphics[width=\linewidth]{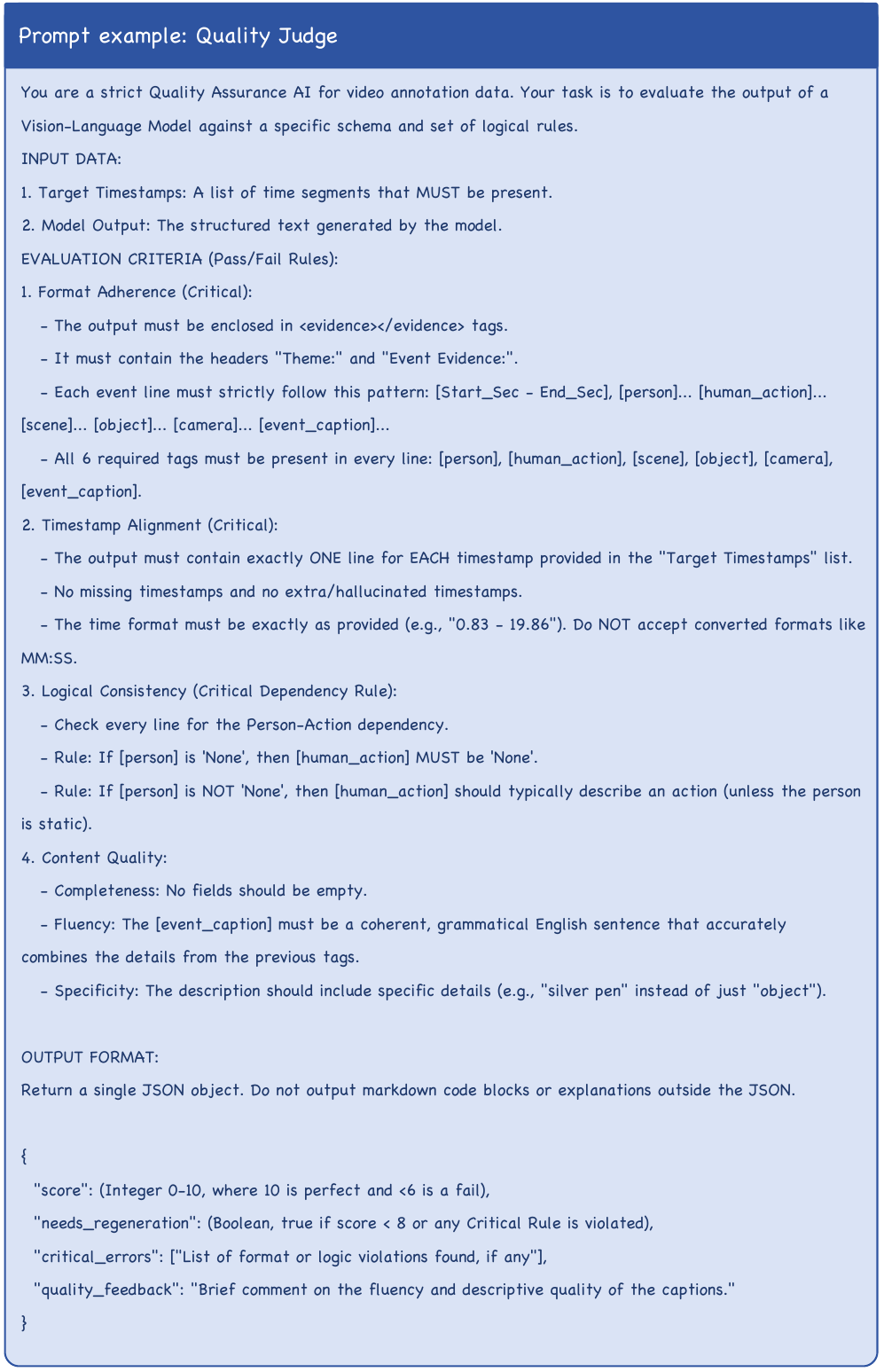}
    \end{minipage}
    \caption{Quality Judge prompt used in the data curation pipeline to validate generated evidence.}
    \label{fig:prompt_judge}
\end{figure}

\begin{figure}[H]
\centering
    \begin{minipage}[b]{1.0\linewidth}
        \includegraphics[width=\linewidth]{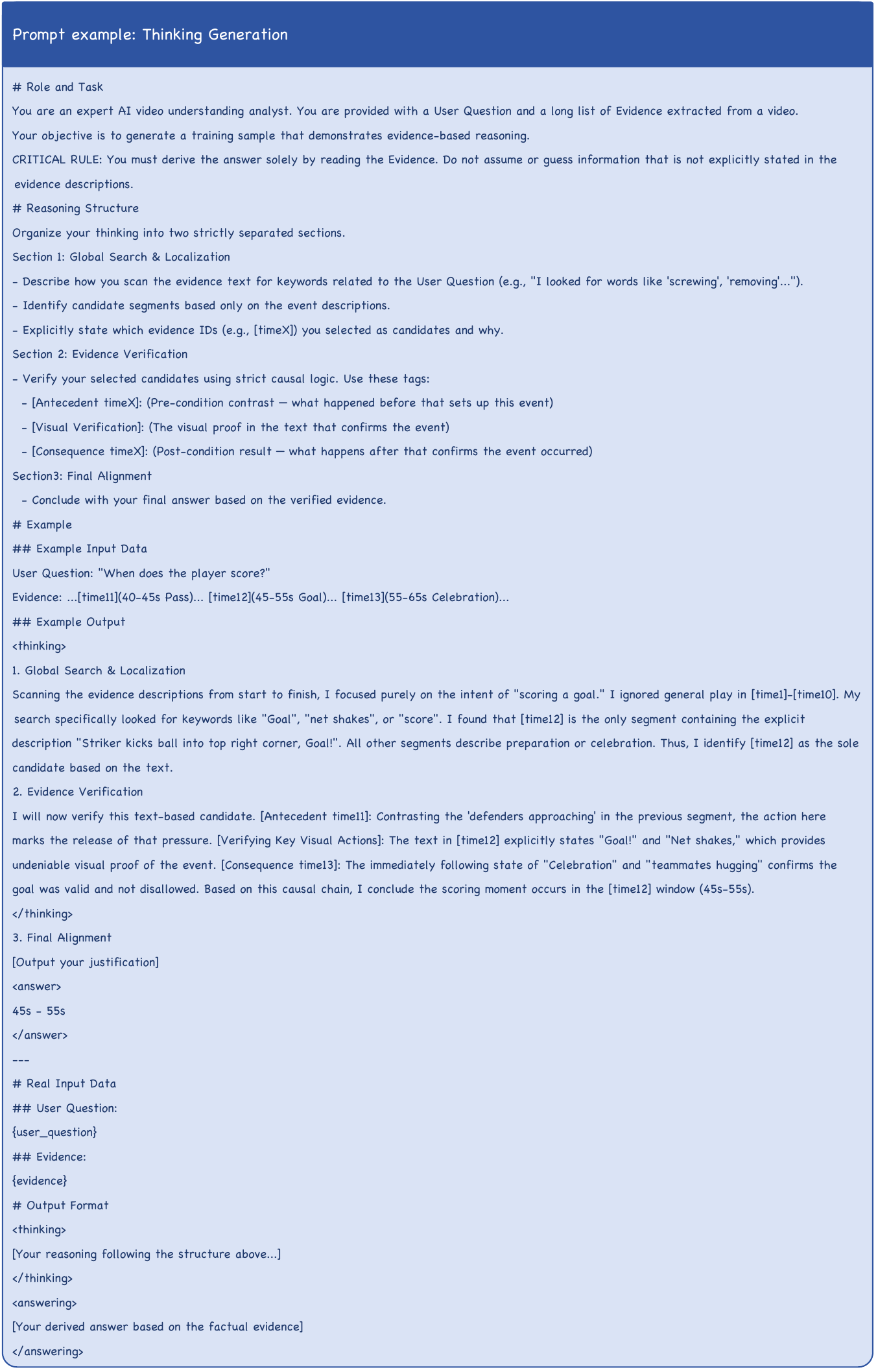}
    \end{minipage}
    \caption{Thinking Generation prompt used in the data curation pipeline to produce evidence-grounded reasoning traces.}
    \label{fig:prompt_think}
\end{figure}

\begin{figure}[H]
\centering
    \begin{minipage}[b]{1.0\linewidth}
        \includegraphics[width=\linewidth]{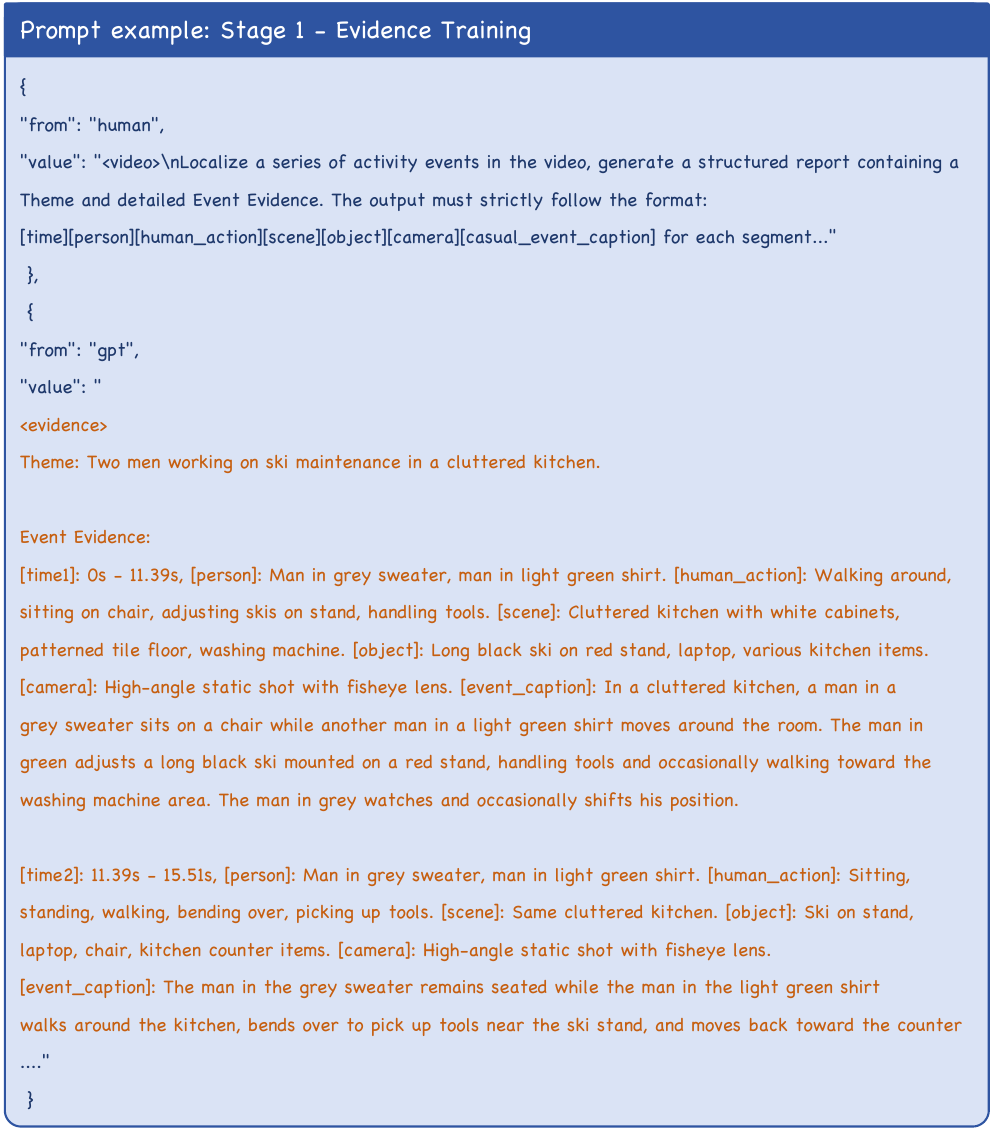}
    \end{minipage}
    \caption{Prompts for Stage 1 evidence training.}
    \label{fig:prompt_stage1}
\end{figure}

\begin{figure}[H]
\centering
    \begin{minipage}[b]{1.0\linewidth}
        \includegraphics[width=\linewidth]{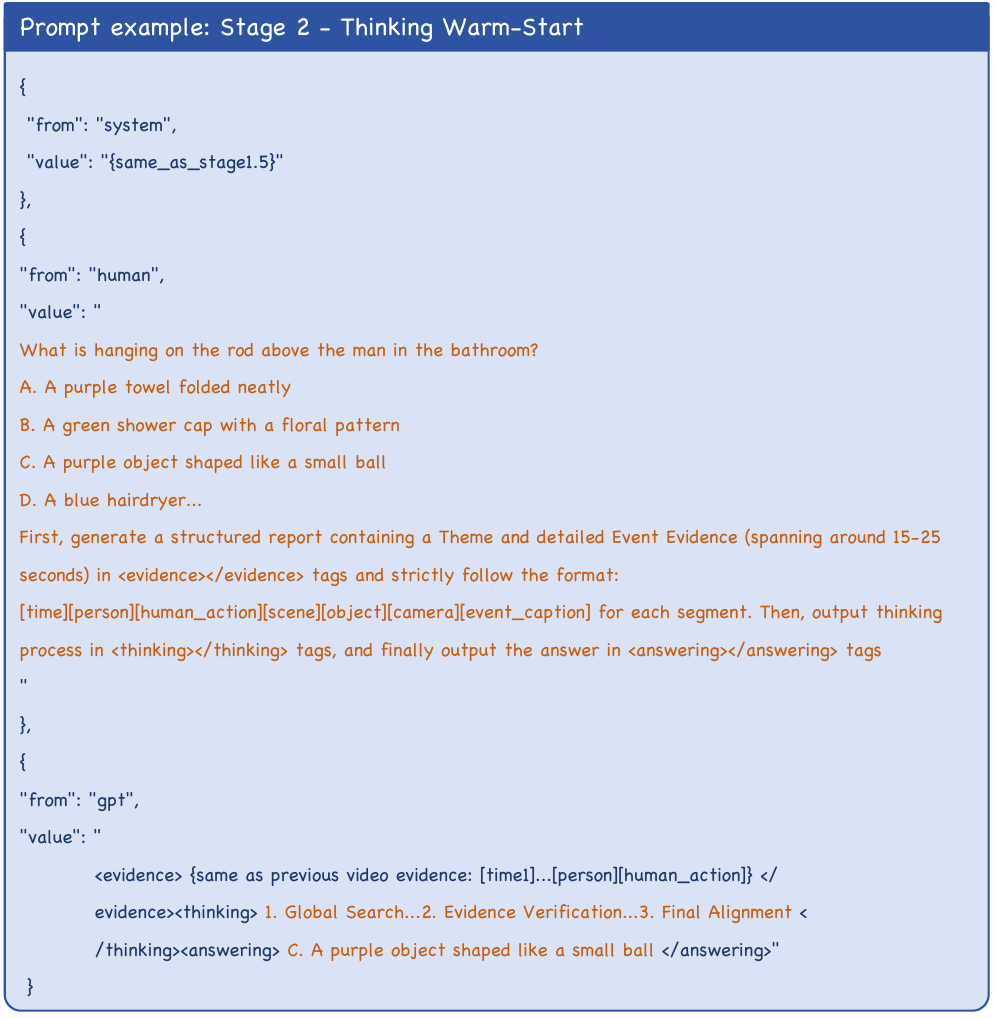}
    \end{minipage}
    \caption{Prompts for Stage 2 thinking warm-start. Orange parts mean these parts are different from the previous parts.}
    \label{fig:prompt_stage2}
\end{figure}

%% file: tables/reward_hparams.tex
%
\begin{table}[H]
\centering
\resizebox{0.72\linewidth}{!}{%
\begin{tabular}{@{}llc@{}}
\toprule
\textbf{Symbol} & \textbf{Meaning} & \textbf{Value} \\ \midrule
$L_{max}$ & hard maximum response length (tokens) & $4096$ \\
$L_{buffer}$ & soft buffer below $L_{max}$ (tokens) & $2500$ \\
$L_{target}$ & $L_{max} - L_{buffer}$, start of the length penalty & $1596$ \\ \midrule
$w^{\text{acc}}$ & Fixed Weight (Accuracy), over (fmt, IoU, MC, len) & $(1,2,2,1)$ \\
$w^{\text{def}}$ & weight vector of the GRPO, DAPO and GDPO runs & $(1,1,1,1)$ \\
$L_{\text{target}}^{\text{lag}}$ & length target of the Lagrangian baseline & $2500$ \\ \midrule
$\tau_{\sigma}$ & variance threshold for the live set & $10^{-6}$ \\
$\eta$ & min-norm floor before the uniform fallback & $10^{-3}$ \\
$\tau_{\mathrm{fw}}$ & Frank--Wolfe weight tolerance & $10^{-6}$ \\
$T_{\max}$ & Frank--Wolfe iteration cap & $50$ \\ \bottomrule
\end{tabular}
}
\vspace{1em}
\caption{Reward and solver hyperparameters of the Stage-3 run. Weights are given over
the four implemented rewards; since a prompt activates only one of the two task
metrics, a group sees $(w_{\text{fmt}}, w_{\text{task}}, w_{\text{len}})$, which is
$(1,2,1)$ for Fixed Weight (Accuracy) and $(1,1,1)$ for the default runs.}
\label{tab:reward_hparams}
\end{table}

%% file: tables/tempcompass_eventbench.tex
\begin{table}[t]
\centering

\resizebox{0.8\linewidth}{!}{%
\begin{tabular}{@{}l|ccc|ccc@{}}
\toprule
\multirow{2}{*}{\textit{\textbf{Model}}} & \multicolumn{3}{c|}{\textbf{TempCompass}} & \multicolumn{3}{c}{\textbf{EventBench}} \\ \cmidrule(l){2-7}
 & Overall & \begin{tabular}[c]{@{}c@{}}Event\\Order\end{tabular} & \begin{tabular}[c]{@{}c@{}}Attr.\\Change\end{tabular} & Overall & \begin{tabular}[c]{@{}c@{}}Temporal\\Reason.\end{tabular} & \begin{tabular}[c]{@{}c@{}}Causal\\Reason.\end{tabular} \\ \midrule
Qwen3-VL-4B-Instruct & 68.71 & 74.76 & 70.40 & 62.71 & 64.51 & 66.87 \\
Qwen3-VL-4B-Thinking & 68.60 & 73.20 & 69.00 & 60.44 & 62.57 & 65.22 \\
\textbf{PARSE-4B} & \textbf{70.64} & \textbf{76.96} & \textbf{73.85} & \textbf{63.48} & \textbf{65.42} & \textbf{69.47} \\ \bottomrule
\end{tabular}
}

\vspace{1em}
\caption{Results on temporal and event reasoning benchmarks. TempCompass evaluates temporal perception; EventBench assesses event-level understanding including temporal and causal reasoning sub-tasks.}
\label{tab:tempcompass_eventbench}
\end{table}

%% file: tables/weight_sweep.tex
%
\begin{table}[H]
\centering
\resizebox{0.92\linewidth}{!}{%
\begin{tabular}{@{}lccccc@{}}
\toprule
\textbf{Run ($G=16$)} & $\mathbf{w_{\text{len}}/w_{\text{task}}}$ & \textbf{R1@0.3} & \textbf{R1@0.5} & \textbf{R1@0.7} & \textbf{Len} \\ \midrule
Fixed Weight (Accuracy), $w{=}(1,2,0)$            & $0$    & 62.4 & 42.1 & 23.8 & 1180 \\
Fixed Weight (Accuracy), $w{=}(1,2,1)$            & $0.50$ & 63.2 & 42.9 & 24.5 & 931 \\
GRPO, uniform weights $w{=}(1,1,1)$               & $1.00$ & 59.8 & 40.6 & 21.8 & 846 \\
Fixed Weight (Accuracy), $w{=}(1,2,4)$            & $2.00$ & 57.5 & 37.8 & 19.6 & 468 \\ \midrule
\textbf{\method{}}                                & per group & \textbf{63.7} & \textbf{43.4} & \textbf{24.6} & 579 \\ \bottomrule
\end{tabular}
}
\vspace{1em}
\caption{Static reward weights on Charades-TimeLens at $G=16$, written as
$(w_{\text{fmt}}, w_{\text{task}}, w_{\text{len}})$ over the three objectives that can
be live in a single group. Fixed Weight (Accuracy), GRPO and \method{} are the rows of
the same name in Table~\ref{tab:ablation}. The other two runs take the Fixed Weight
(Accuracy) vector and set $w_{\text{len}}$ to zero or to four times its tuned value,
leaving $w_{\text{fmt}}$ and $w_{\text{task}}$ alone, so the four static runs trace a
single ratio. Output length falls monotonically in $w_{\text{len}}/w_{\text{task}}$
across them, so the column orders the family.}
\label{tab:weight_sweep}
\end{table}

%% file: tables/seed_variance.tex
\begin{table}[H]
\centering
\resizebox{0.78\linewidth}{!}{%
\begin{tabular}{@{}lcccc@{}}
\toprule
\textbf{Run ($G=16$)} & \textbf{R1@0.3} & \textbf{R1@0.5} & \textbf{R1@0.7} & \textbf{Len} \\ \midrule
GRPO                    & $59.8 \pm 0.30$ & $40.6 \pm 0.30$ & $21.8 \pm 0.40$ & $846 \pm 12$ \\
Fixed Weight (Accuracy) & $63.2 \pm 0.30$ & $42.9 \pm 0.30$ & $24.5 \pm 0.40$ & $931 \pm 14$ \\
\textbf{\method{}}      & $\mathbf{63.7 \pm 0.25}$ & $\mathbf{43.4 \pm 0.25}$ & $\mathbf{24.6 \pm 0.35}$ & $\mathbf{579 \pm 9}$ \\ \bottomrule
\end{tabular}
}
\vspace{1em}
\caption{Spread across three random seeds on Charades-TimeLens. The means are the
values reported in Table~\ref{tab:ablation}, and the table adds the standard
deviation over the seeds.}
\label{tab:seed_variance}
\end{table}

%% file: tables/measured_weights.tex
%
%
\begin{table}[H]
\centering
\resizebox{0.72\linewidth}{!}{%
\begin{tabular}{@{}lc@{}}
\toprule
\textbf{Quantity} & \textbf{Value} \\ \midrule
$\alpha^*_{\text{fmt}}$, three live objectives  & $0.15 \pm 0.11$ \\
$\alpha^*_{\text{task}}$, three live objectives & $0.39 \pm 0.06$ \\
$\alpha^*_{\text{len}}$, three live objectives  & $0.46 \pm 0.07$ \\
$\alpha^*$, two live objectives (self-check)    & $0.50 \pm 0.00$ \\ \midrule
R1@0.5 over GRPO, matched steps up to $100$     & $+0.7$ to $+2.1$ \\
Mean length against GRPO, same steps            & $-8\%$ to $-19\%$ \\ \bottomrule
\end{tabular}
}
\vspace{1em}
\caption{Weights the solver returns, measured over the PARSE-4B run. The first
three rows give the mean and the standard deviation of $\boldsymbol{\alpha}^*_q$
across the groups where all three objectives are live; the means sum to one, and
uniform weighting would put $1/3$ on each. The fourth row is the value
\cref{sec:live_objectives} gives for two live objectives, measured on real groups
rather than assumed. The last two rows compare
the two runs checkpoint by checkpoint at the same training step, on
Charades-TimeLens, with the range covering the checkpoints through step $100$.}
\label{tab:measured_weights}
\end{table}